\documentclass{article}

\usepackage[utf8]{inputenc} % allow utf-8 input
\usepackage[T1]{fontenc}    % use 8-bit T1 fonts
\usepackage{hyperref}       % hyperlinks
\usepackage{url}            % simple URL typesetting
\usepackage{booktabs}       % professional-quality tables
\usepackage{amsfonts}       % blackboard math symbols
\usepackage{nicefrac}       % compact symbols for 1/2, etc.
\usepackage{microtype}      % microtypography
\usepackage{xcolor}         % colors

\usepackage{graphicx}
\usepackage{mathrsfs}
\usepackage{amsfonts, amsmath,dsfont,amssymb,amsthm,stmaryrd,bbm}
\usepackage{mathtools}
\usepackage{caption}
\usepackage{subcaption}
\usepackage{enumitem}
\setlist{nolistsep}
\setlist{nosep}
\usepackage{wrapfig}
\usepackage{algorithm}
\usepackage{algorithmic}
\usepackage{todonotes}
\usepackage{xspace}
\usepackage{comment}
\usepackage[a4paper, margin=1in]{geometry}

 \newtheorem{conj}{Conjecture}
\newtheorem{thm}{Theorem}

\newtheorem{rmk}{Remark}

\newtheorem{lemma}{Lemma}

\newtheorem{prop}[conj]{Proposition}

\newtheorem{defn}[conj]{Definition}

\newtheorem{assumption}{Assumption}

\newcommand{\f}[1]{\boldsymbol{#1}}
\newcommand{\bb}[1]{\mathbb{#1}}
\newcommand{\fl}[1]{\mathbf{#1}}
\newcommand{\ca}[1]{\mathcal{#1}}

\newcommand{\s}[1]{\mathsf{#1}}

\newcommand{\lr}[1]{{\left|\left|#1\right|\right|}}

\def\cT{\mathcal{T}}

\title{Online Matrix Completion: A Collaborative Approach with Hott Items}

\author{Dheeraj Baby \thanks{Both authors contributed equally to this work.} \\ \\dheeraj@ucsb.edu \\ UC Santa Barbara \and Soumyabrata Pal {*} \\ \\soumyabratapal13@gmail.com \\ Adobe Research}

\date{}

\begin{document}

\maketitle

\begin{abstract}
    We investigate the low rank matrix completion problem in an online setting with $\s{M}$ users, $\s{N}$ items, $\s{T}$ rounds, and an unknown rank-$r$ reward matrix $\fl{R}\in \bb{R}^{\s{M}\times \s{N}}$. This problem has been well-studied in the literature  \cite{jain2022online,dadkhahi2018alternating,sen2017contextual,zhou2020stochastic}  and has several applications in practice. 
    In each round, we recommend $\s{S}$ carefully chosen distinct items to every user and observe noisy rewards. In the regime where $\s{M},\s{N}\gg \s{T}$, we propose two distinct computationally efficient algorithms for recommending items to users and analyze them under the benign \textit{hott items} assumption 
    %i.e. the convex hull of the latent item embeddings have $r+1$ corners including zero -
1) First, for $\s{S}=1$, under additional incoherence/smoothness assumptions on $\fl{R}$, we propose the phased algorithm \textsc{PhasedClusterElim}. Our algorithm obtains a near-optimal per-user regret of $\widetilde{O}(\s{N}\s{M}^{-1}(\Delta^{-1}+\Delta_{\s{hott}}^{-2}))$ where $\Delta_{\s{hott}},\Delta$ are problem-dependent gap parameters with $\Delta_{\s{hott}}\gg \Delta$ almost always.  
%This improves significantly on the regret guarantees of greedy algorithms introduced in \cite{jain2022online,sen2017contextual}.
2) Second, we consider a simplified setting with $\s{S}=r$ where we make significantly milder assumptions on $\fl{R}$. Here, we introduce another phased algorithm, \textsc{DeterminantElim}, to derive a regret guarantee of $\widetilde{O}(\s{N}\s{M}^{-1/r}\Delta_\s{det}^{-1}))$ where $\Delta_{\s{det}}$ is another problem-dependent gap. Both algorithms crucially use collaboration among users to jointly eliminate sub-optimal items for groups of users successively in phases, but with distinctive and novel approaches.
\end{abstract}

\section{Introduction}

Collaborative Filtering via matrix completion \cite{Koren:2008} is a fundamental framework for large-scale recommendation systems that cater to millions of users and items jointly. %The goal of the recommendation system is to quickly learn the taste of users via collaboration of users i.e. aggregate their ratings/feedback and exploiting some structural similarity among users. 
Usually collaborative filtering is an online problem since the recommendation system can update its suggestions based on the history of observed ratings from users. 
%However, existing literature in matrix completion is mostly provided in the offline setting \cite{candes2009exact,jain2013low,abbe2020entrywise}. 
Thus, there is a clear tension between exploration and exploitation - the system needs to learn the user preferences accurately through a diverse set of recommendations and yet recommend meaningful likeable items quickly.
%It is imperative that the time necessary for the recommendation system to generate meaningful recommendations (also known as cold-start period) is small - this can be achieved only through collaboration among users. 
Recently, to study the aforementioned exploration/exploitation dilemma theoretically, \cite{jain2022online,https://doi.org/10.48550/arxiv.2301.07040} have modelled the problem on lines of online matrix completion that is equivalent to the well-known Multi-Armed Bandit (MAB) optimization \cite{lattimore2020bandit} for multiple agents (representing users) jointly over a common set of items (representing arms).
%and characterized guarantees for regret - a notion of loss in these settings. 

We consider the same setting of online matrix completion under bandit feedback with $\s{M}$ users, $\s{N}$ items, $\s{T}$ rounds, and an associated low rank reward matrix $\fl{R}\in \bb{R}^{\s{M}\times \s{N}}$. This setup was first studied theoretically in \cite{jain2022online} and a close variant in \cite{sen2017contextual} - at each round, \emph{one} item is recommended to every user. Note that the recommended item in a particular round can be different for each user and we only observe noisy rewards for the recommended item-user pairs. 
The main difficulty in obtaining optimal algorithms in this setting is the latent structure which makes it challenging to characterize sharp confidence widths on non-uniform sampling of items.
On the other hand, greedy algorithms exploiting the low rank of reward matrix were proposed in \cite{jain2022online},\cite{sen2017contextual} that achieved a regret guarantee of $\widetilde{O}(\s{N}\s{M}^{-1}\Delta^{-2})$ 
%($\widetilde{O}(\cdot)$ hides logarithmic factors in $\s{M},\s{N},\s{T}$)
%\footnote{The regret guarantee of the greedy algorithm provided in \cite{jain2022online} is worst-case and scales as $\widetilde{O}((\s{N}\s{M}^{-1})^{1/3}\s{T}^{2/3})$}  
where $\Delta$ is the minimum sub-optimality gap across users and items (Definition \ref{gap:standard}).  Exploiting collaboration among users i.e. aggregating their feedback lead to a significantly smaller effective number of items per user - $O(\s{N}\s{M}^{-1})$. 
However, because of the greedy nature, the achievable regret guarantee is sub-optimal in its dependence on the gap parameter $\Delta$. 

To alleviate this issue, an optimal algorithm with a regret guarantee of $\widetilde{O}(\s{N}\s{M}^{-1}\Delta^{-1})$ was proposed \cite{jain2022online} only for rank-$1$ reward matrices. The key idea was that for rank-$1$ setting, users can be partitioned into two latent clusters - users in the same cluster have the same preference ordering of items while users in different clusters have
the exact opposite preference ordering of items. While this property was crucially used for algorithm design in rank-$1$ case, it does not extend for higher rank. 
In fact, even when the  reward matrix is rank-$2$, the  possible orderings of items across users is polynomially large in $\s{N}$ \cite{asteris2014nonnegative}. To get around this, in \cite{maillard2014latent,https://doi.org/10.48550/arxiv.2301.07040,gentile2014online,lee2023context}, the authors provided optimal algorithms by  assuming that users can be partitioned into a small number of latent clusters - users in same cluster have identical preferences across items. However, such an assumption can be very restrictive in practice. On the flip side,
\cite{dadkhahi2018alternating,zhou2020stochastic} studied the same problem with a weak low rank assumption on the reward matrix - they proposed computationally efficient algorithms but with no theoretical guarantees.

For tractability of analysis, we adopt a relatively weak hott items (aka separability) assumption (A\ref{assum:hott}) widely used in literature - for instance, in non-negative matrix factorization (NMF) \cite{arora2012computing,recht2012factoring,donoho2003does}, bandit optimization \cite{sen2017contextual,kveton2017stochastic} and mixture models \cite{pal2022learning}. The assumption only posits that that the best (most preferred) item of every user can belong to a small set of  particular items called \textit{hott} items \footnote{An example of hott items - in a large library of films, only a small number have the highest of ratings provided by users.}. Consider a weak notion of  latent user clusters based on the best item - users with the same best item belong to same cluster. Our assumption imply that number of such clusters is small.
However,
unlike the rank-$1$ setting, users in the same cluster can have vastly different ordering of remaining items. 
%Thus, we model situations such as a handful of movies being the most popular among a large library of films. 
In the seminal work on topic modelling via LDA, \cite{blei2003latent} showed that such an assumption is satisfied approximately by several real-world recommendation datasets.
%the assumption also has been motivated and exploited in bandit optimization problems for online recommendation and advertising in \cite{sen2017contextual, kveton2017stochastic}. 

%Also note that, Assumption \ref{assum:hott} does not imply anything about the expected reward values across any pair of users - even for the best item.

\noindent \textbf{Techniques and Contributions:} In this work, we make two contributions related to online matrix completion for rank $>1$ reward matrices - the goal is to improve regret guarantees of greedy algorithms proposed in \cite{jain2022online,sen2017contextual}.
%We describe two distinctive algorithms and analyze their theoretical guarantees under different sets of assumptions. 
\begin{itemize}[leftmargin=*,noitemsep,nolistsep]
\item First, if number of hott items is small (\textbf{A\ref{assum:hott}}) and with additional standard incoherence/smoothness assumptions (\textbf{A\ref{assum:matrix2}}) on reward matrix $\fl{R}$ with rank $r>1$, we propose Alg. \textsc{PhasedClusterElim}.  Initially, our algorithm finds a set of $r$ distinct \textit{opinionated users} - users who have very strong preferences for a particular hott item over other items.
\textit{This is easy since we can expect few users to be very opinionated in most applications.} We consider this set of identified \textit{opinionated users} as labels - the goal is to cluster users based on best item and thus associate the true label
\footnote{To elaborate, true label corresponds to the \textit{opinionated} user sharing the same best (hott) item as the non-opinionated user.}
with each of the more challenging \textit{non-opinionated} users. 
Thus, we attempt, in phases, to associate each \textit{non-opinionated} user to at-least one of the labels.
%based on their mutual similarity in latent embeddings.  
In each phase, we assign multiple labels to each  \textit{non-opinionated} user encoding the uncertainty. 
%the labels assigned corresponds to the hott items that can possibly be the best item for the user. 
Since there are $r$ opinionated users, the above label association defines $r$ overlapping groups/clusters of users corresponding to each label. For each group of users, we find a common set of \textit{good} items that are high rewarding. Next, we estimate relevant sub-matrices of the reward-matrix from feedback data - based on the estimates, we refine items for each user group. Algorithm \textsc{PhasedClusterElim} obtains a \textit{near-optimal} regret guarantee described in Theorem \ref{thm:first}. We remark that no prior knowledge of  gaps is required to invoke our algorithms. %here $\Delta_{\s{hott}}$ is another problem dependent gap (Defn. \ref{defn:weird_gap}) which is significantly larger than $\Delta$.  %The first term $\s{N}\s{M}^{-1}\Delta_{\s{hott}}^{-2}$ corresponds to the regret in identifying a set of \textit{opinionated users} (which is small even with a quadratic dependence on $\Delta_{\s{hott}}$) since we only need to  identify one \textit{opinionated user} for each hott item. 
%Note, the second term $\s{N}\s{M}^{-1}\Delta^{-1}$ has a optimal dependence on $\s{M},\s{N},\Delta$.

\item Our second contribution is to propose a theoretically sound algorithm under significantly milder assumption of non-negative reward matrix (\textbf{A2}) satisfying the hott items property.
%In addition to (\textbf{A1}), we make a mild assumption that latent embeddings of all users and items have non-negative entries (\textbf{A2}). 
To this end, following  \cite{kveton2017stochastic}, we consider a simplified setting where $r$ items are recommended to every user at each round. The set of items recommended at each round can vary across users. We wish to ensure that there is at-least one item in the recommended set that is high rewarding \footnote{This setting models situations where a handful of movies are recommended by a streaming website and the user is satisfied if at-least one of the recommendations is of interest to the user}.
%and a corresponding notion of regret which is slightly simplified. 
%Note that, for every user, it is easy to identify the best item  once we have converged to a shortlist of some $r$ items comprising the best item.
The hott items assumption (\textbf{A1}) ensures that the unique highest rewarding set of $r$ items that can be commonly recommended to any user is the set of $r$ hott items. 
Our proposed algorithm \textsc{DeterminantElim}
eliminates subsets of $r$ items in
successive phases and achieves a regret guarantee described in Theorem \ref{thm:second}. For $r=1$ (rank one reward matrix), our regret guarantees recover the results in \cite{jain2022online} and are thus optimal.
\end{itemize}

\noindent \textbf{Technical Challenges:} Here, we  highlight the challenges in analysis of Alg. \textsc{PhasedClusterElim}. Recall that at beginning, we find a set of $r$ \textit{opinionated users} with distinct hott items as best item and use them as labels. The \textit{opinionated users} have strong preferences towards their corresponding best item and are thus easy to identify.
For labelling \textit{non-opinionated} users, the challenge is that their preferences are more spread out - this makes labelling difficult. 
Our algorithm, in each phase, assigns the set of \textit{non-opinionated} users into $r$ overlapping groups. 
%(each user can belong to multiple groups) and estimates large sub-matrices of the reward matrix corresponding to each group of users and a corresponding joint set of good items. 
%We first prove that the best item of each user must be same as one of the \textit{opinionated users} with which it shares a group.
There are three key steps 1) First, we show that each group of users (in any phase) is a superset of a cluster of users all having the same best item. 
%We construct $r$ such groups where each user can be in multiple groups encoding the uncertainty of the user's best item 
2) Second, we show that in every phase the joint set of good items for each user group (corresponding to a label) is high rewarding for all users in the group 3) Fnally, we show that the best hott item for every user survives in the union of the joint set of good items for the groups that the user belongs to. Therefore, we refine the set of items gradually with each phase (jointly across many users) and obtain our guarantees. 
Later, in Section \ref{sec:difficult}, we also demonstrate with an illustrative example why simple item elimination strategies fail - necessitating our novel approach.
%The above properties allow us (in each phase) to construct $r$ sub-matrices of the reward matrix $\fl{R}$ (possibly with overlapping rows and columns) and estimate them with uniform exploration. 

For the analysis of Alg. \textsc{DeterminantElim}, we use the fact that the hott items are, in some sense, the unique highest rewarding set of $r$ items that can be recommended to any user. Hence, we track subsets of $r$ items and eliminate them succesively to converge to the set of $r$ hott items.
This was also observed in \cite{kveton2017stochastic}  where a similar problem was studied with the goal of finding the largest entry of a low rank matrix in an online fashion. However, in their setting, it was possible to choose both row and column of reward matrix in each round. 
However, we need to recommend to every user $r$ items at each round. \textit{This is challenging because the difficulty of identifying preferred items can vary significantly across users.}
%i.e. for every row, we need to select $r$ columns at  each round.
Here, we 
%need to ensure that for each surviving subset of $r$ items, observations are obtained corresponding to a significant number of common users in each phase. This is required in order to 
jointly estimate how \textit{good} each group of $r$ items is  and eliminate sub-optimal groups of $r$-items based on the estimates - these estimates are based on determinant of relevant sub-matrices. 
%A trivial approach is to recommend all surviving sets of $r$ items a certain number of times - this is clearly sub-optimal.
Our key idea (Lemma \ref{lem:sampling}) is to exploit redundancy in recommending overlapping sets of $r$ items at every round and characterize the sufficient rounds in each phase for ensuring nice concentration properties of our estimates (Lemma \ref{lem:conc}).  Due to space restrictions, a description of other loosely related work has been delegated to Appendix.
%Putting together, we are able to bound the regret guarantees.  

 %In Sec. \ref{sec:problem}, we formulate precisely our problem setting and discuss in details the main results.  In Sec. \ref{sec:difficult}, we provide a high level overview of Algorithm \textsc{PhasedClusterElim} and its analysis (proofs delegated to Appendix). In Sec. \ref{sec:simple}, we provide a high level overview of Algorithm \textsc{DeterminantElim} and its analysis (proofs delegated to Appendix).

\section{Problem Formulation and Results}\label{sec:problem}

\subsection{Formulation}
\noindent \textbf{Notations:} We write $[m]$ to denote the set $\{1,2,\dots,m\}$. 
% For a vector $\fl{v}\in \bb{R}^m$, $\fl{v}_i$ denotes the $i^{\s{th}}$ element; for any set $\ca{U}\subseteq [m]$, let $\fl{v}_\ca{U}$ denote the vector $\fl{v}$ restricted to the indices in $\ca{U}$. 
$\fl{M}_i$ denotes the $i^{\s{th}}$ row of $\fl{M}$ and
$\fl{M}_{ij}$ or $\fl{M}_{i,j}$  denotes the $(i,j)$-th element of matrix $\fl{M}$. For any set $\ca{U}\subset [m],\ca{V}\subset [n]$,  $\fl{M}_{\ca{U},\ca{V}}$ denotes the matrix $\fl{M}$ restricted to the rows in $\ca{U}$ and columns in $\ca{V}$. Similarly, $\fl{M}_{\ca{U}}$ denotes the matrix $\fl{M}$ restricted to the rows in $\ca{U}$.
We write $\s{det}(\fl{M})$ to denote the determinant of a matrix $\fl{M}$. 
%We denote by $\fl{M}(\ca{I},:)$ its sub-matrix of $k$ rows $\ca{I}\in[\s{M}]^k$. Similarly, we denote by $\fl{M}(:,\ca{J})$ its sub-matrix of $k$ columns $\ca{J}\in[\s{N}]^k$. 
We will assume that the elements in sets $\ca{I},\ca{J}$ will be ordered in ascending order. We will refer as $d-$row (analogously $d-$ column) a set of $d$ distinct rows (columns respectively) of a matrix.
We denote the set $\ca{S}_d \equiv \{\fl{v}\in [0,1]^d\mid \lr{\fl{v}}_1 \le 1\}$ to be the $d$-dimensional simplex; $\ca{S}_{n,d}$ denotes the set of $n \times d$ matrices whose rows belong to $\ca{S}_d$ i.e. $\ca{S}_{n,d}=\{\fl{M}\in [0,1]^{n \times d}\mid \fl{M}_i \in \ca{S}_d\}$. Let $\Pi_k(\ca{A})$ denote the set of all $k$-subsets of set $\ca{A}$ with distinct elements. 
We use $\bb{E}X$ to denote the expectation of a random variable $X$.  $\widetilde{O}(\cdot)$ hides logarithmic dependencies on $\s{M},\s{N},\s{T}$. Notations such as $\widetilde{\fl{Q}}$ denotes an estimate of a quantity $\fl{Q}$.

%Also, let $\|\fl{A}\|_{2\rightarrow \infty}$ be the maximum $\ell_2$ norm of the rows of  $\fl{A}$ and $\|\fl{A}\|_{\infty}$ be the absolute value of the largest entry in $\fl{A}$. We write $\bb{E}X$ to denote the expectation of a random variable $X$.

Consider a set of $\s{M}$ users, $\s{N}$ items and a horizon of $\s{T}$ rounds. We have an associated unknown rank-$r$ expected reward matrix $\fl{R}\in [-1,1]^{\s{M}\times \s{N}}$ 
such that $\fl{R}=\fl{U}\fl{V}^{\s{T}}$ where $\fl{U}\in \bb{R}^{\s{M}\times r}$, $\fl{V}\in \bb{R}^{\s{N}\times r}$ denotes the user embedding matrix and item embedding matrix respectively.
At each round $t\in [\s{T}]$, for each user $u\in [\s{M}]$, $\s{S}$ carefully chosen distinct items $\{\rho_u(t,j)\}_{j=1}^{\s{S}}\in [\s{N}]$ are recommended and we observe $\s{S}$ rewards denoted by the random variables $\{\fl{P}^{(t)}_{u\rho_u(t,j)}\}_{j=1}^{\s{S}}$ such that 
\begin{align}\label{eq:obs}
\fl{P}^{(t)}_{u\rho_u(t,j)} = \fl{R}_{u\rho_u(t,j)}+\fl{E}^{(t)}_{u\rho_u(t,j)}.
\end{align}
Here $\fl{E}^{(t)}_{u\rho_u(t,j)}$ denotes the additive noise that are i.i.d zero-mean sub-Gaussian
random variables with variance proxy at most $\sigma^2$. 
We consider the practical regime where $\s{M},\s{N}\gg \s{T}$ i.e. the number of users and items are much larger than the number of rounds. Let ordering 
 $\pi_u:[\s{N}]\rightarrow [\s{N}]$ sort the rewards for each user $u \in [\s{M}]$ in descending order, i.e., for any $i<j$, $\fl{R}_{u\pi_u(i)} \ge \fl{R}_{u\pi_u(j)}$. 
We make the following  assumption:

\begin{assumption}[\textbf{A1} Hott items and vectors]\label{assum:hott}

We assume that there exists a set of $r$ unknown distinct ordered indices $\ca{A}\subseteq [\s{N}], \left|\ca{A}\right|=r$ such that all the item embedding vectors $\{\fl{V}_i\}_{i\in [\s{N}]}$ lie within the convex hull of $\{\fl{V}_i\}_{i\in \ca{A}}\cup\{\fl{0}\}$.
 We will call the items corresponding to indices in $\ca{A}$ hott items and the rows in $\fl{V}_{\ca{A}}$ to be hott vectors. 
\end{assumption}

As mentioned before, Assumption \ref{assum:hott}, also sometimes known as separability assumption in existing literature has been used widely. In fact, as pointed out in \cite{donoho2003does,arora2012computing}, an approximate separability condition is regarded as
a fairly benign assumption that holds in many practical contexts in machine learning such as LDA in information retrieval \cite{blei2003latent}. 
We first show that with Assumption \ref{assum:hott}, for each user $u\in [\s{M}]$, the highest and lowest rewarding items $\pi_u(1),\pi_u(\s{N})\in \ca{A}$ are hott items.
\begin{lemma}\label{lem:prelim_1}
For any $u\in [\s{M}]$, it must happen that 
$\pi_u(1),\pi_u(\s{N}) \in \ca{A}$.
\end{lemma}

Thus, we can partition the users $\ca{T}^{(1)},\ca{T}^{(2)},\dots,\ca{T}^{(r)}$ into $r$ clusters such that $\forall i\in [r],\ca{T}^{(i)} \triangleq \{u\in [\s{M}]\mid \pi_u(1)=\ca{A}_i\}$ is the set of users whose best item is the $i^{\s{th}}$ hott item.
For some $\kappa$, we can write the minimum cluster size to be $\min_{i\in [r]}\left|\ca{T}^{(i)}\right| \ge \kappa \s{M}r^{-1}$.
We will now consider two distinct (although related) notions of regret under different additional assumptions on the reward matrix $\fl{R}$:

\noindent \textbf{Setting 1  (General Regret):}
We consider $\s{S}=1$ i.e. at every round $t\in [\s{T}]$, for each user $u\in [\s{M}]$, a single item $\rho_u(t)$ is recommended to user $u$.  The objective is to design an algorithm (possibly randomized) for the recommendation policy so that  the expected regret defined below is minimized:
\begin{align}\label{eq:general}
\s{Reg}\triangleq \frac{1}{\s{M}} \Big(\sum_{t\in[\s{T}]}\sum_{u\in[\s{M}]}\mathbf{R}_{u\pi_u(1)}- \sum_{t\in[\s{T}]}\sum_{u\in[\s{M}]}\mathbf{R}_{u\rho_u(t)}\Big).
\end{align}

We make the following additional assumptions on the expected reward matrix $\fl{R}$:
%In order to provide near-optimal theoretical guarantees on the regret notion defined in eq. \eqref{eq:general}, we make the following assumptions on the reward matrix $\fl{R}$:

\begin{assumption}[\textbf{A2}]\label{assum:simplicity}
 We will assume $r,\sigma,\kappa$ are positive constants that do not scale with $\s{M},\s{N},\s{T}$.  
 % ratio of  cluster sizes $\kappa \triangleq \max_{i,j} \frac{\left|\ca{T}^{(i)}\right|}{\left|\ca{T}^{(j)}\right|}=O(1)$.   
 \end{assumption}

Note that Assumption \ref{assum:simplicity} is made only for simplicity of exposition/analysis and is not necessary for our theoretical guarantees or the algorithm. It is routine to generalize our results for a non-constant $r,\sigma,\kappa$ and leads to a polynomial dependence on the parameters in the regret guarantee. 

\begin{assumption}[\textbf{A3} Offline Matrix Completion Oracle]\label{assum:matrix2}
There exists an offline low rank Matrix Completion Oracle (MCO) denoted by  $\mathcal{O}(\eta,R,C)$ which takes a sub-optimality level $\eta$, a set of row indices $R$ and a set of column indices $C$. Consider a low rank matrix $\mathbf{Q}$ of constant rank $r$. The matrix completion oracle samples indices from the sub-matrix $\mathbf{Q}_{R,C}$ such that for each row in $R$, $O_{\sigma,r}(max(1,C/R)/\eta^2)$ noisy samples (with noise variance $\sigma^2$) are obtained restricted to the columns in $C$. 
%operates in $O(1/\eta^2)$ rounds. At each round, the oracle samples indices from the sub-matrix $\mathbf{Q}_{R,C}$ such that at-least one column is sampled, for each row in $R$. 
Then noisy observations of the sampled entries are obtained as per Eq.\eqref{eq:obs} and subsequently, the oracle MCO computes an estimate $\mathbf{\widehat Q}_{R,C}$ of the matrix $\mathbf{Q}_{R,C}$ such that $\| \mathbf{\hat Q}_{R,C} - \mathbf{Q}_{R,C} \|_\infty \le \eta$ with high probability. The dependence of the sample complexity on the noise variance and rank are hidden with the $O(\cdot)$ notation for simplicity.
\end{assumption}

%\begin{assumption}[\textbf{A3} Assumptions on reward matrix $\fl{R}$]\label{assum:matrix2}
%We assume that $\fl{R}$ with SVD decomposition $\fl{R}=\widetilde{\fl{U}}\f{\Sigma}\widetilde{\fl{V}}^{\s{T}}$ satisfies the following  properties 1) (Condition Number) $\fl{R}$ has rank $r$ and has non zero singular values $\lambda_1>\lambda_2 > \dots > \lambda_{r}$ with $\lambda_1/\lambda_{r}=O(1)$ 2) ($\mu$-incoherence)  $\lr{\widetilde{\fl{U}}}_{2,\infty}\le \sqrt{\mu r/\s{N}}$ and $\lr{\widetilde{\fl{V}}}_{2,\infty}\le \sqrt{\mu r/\s{M}}$ for some $\mu=O(1)$. 3) (Subset Strong Smoothness (a)) For some constant $\beta>0$ and for any subset of indices $\ca{S}\subseteq [\s{N}], \ca{S}= \ca{T}^{(j)}$ (corresponding to some cluster of users), we must have $\fl{x}^{\s{T}}\widetilde{\fl{U}}_{\ca{S}}^{\s{T}}\widetilde{\fl{U}}_{\ca{S}}\fl{x} \ge \beta $ for all unit norm vectors $\fl{x}\in \bb{R}^{r}$.
%4) (Subset Strong Smoothness (b)) For some $\alpha$ satisfying $\alpha \log \s{M}=\Omega(1)$, $\gamma =\widetilde{O}(1)$, for any subset of indices $\ca{S}\subseteq [\s{M}], |\ca{S}| \ge \gamma r$, we must have $\fl{x}^{\s{T}}\widetilde{\fl{V}}_{ \ca{S}}^{\s{T}}\widetilde{\fl{V}}_{ \ca{S}}\fl{x} \ge \alpha \left|\ca{S}\right|/\s{M}$ for all $\fl{x}\in \bb{R}^{r},\lr{\fl{x}}_2=1$. 
%\end{assumption}

There exist several standard algorithms for offline Low Rank Matrix Completion with partial noisy observations. Two well-known approaches are 1) Minimizing MSE with nuclear norm regularizer 2) Alternating Minimization (see \cite{jain2017non} and references therein). Note that the existence of such an offline low rank matrix completion oracle has been demonstrated in \cite{chen2019noisy,abbe2020entrywise,jain2022online} - however the current state of the art theoretical guarantees exist under addtional assumptions on the low rank matrix such as incoherence and condition number. However, from an algorithmic point of view, as we demonstrate in Algorithm \ref{algo:estimate}, a wrapper around them can be easily created. We refer the reader to Section \ref{} for a more detailed theoretical discussion on  Assumption \ref{assum:matrix2} - in particular, we discuss the conditions that need to be satisfied for the low rank matrix that we have at hand so that we can invoke the \textit{current} state of the art offline low rank matrix completion guarantees.

\begin{comment}
Assumption \ref{assum:matrix2}, also proposed in \cite{https://doi.org/10.48550/arxiv.2301.07040} states that matrix $\fl{R}$ must be well-conditioned and incoherent - two properties that are commonly used for invoking theoretical guarantees on offline low rank matrix completion \cite{chen2019noisy,jain2013low}. In essence, Assumption \ref{assum:matrix2} allows us to invoke low rank matrix completion guarantees to suit large sub-matrices of $\fl{R}$.
The subset strong smoothness assumptions on $\fl{U},\fl{V}$ suffice to prove that any sub-matrix of $\fl{R}$ of a reasonable size also satisfies the low condition number and incoherence properties (see \cite{https://doi.org/10.48550/arxiv.2301.07040}). 
%Furthermore, similar smoothness assumptions (Statistical RIP property) were made in \cite{sen2017contextual} where a greedy algorithm was provided for optimizing the regret $\s{Reg}$ based on non-negative matrix factorization. 
\end{comment}
Now, we define some problem-dependent gap parameters:
%In Appendix \ref{app:assum}, we show Assumptions \ref{assum:hott},\ref{assum:matrix2} to be simultaneously true for a large number of matrices.

\begin{defn}[Hott item gaps]\label{defn:weird_gap}
Let $\ca{A}$ be the set of hott items. For any user $u\in [\s{M}]$, define the hott item gap $\Delta_{u} \triangleq \fl{R}_{u\pi_u(1)}-\max_{v\in \mathcal{A}\setminus \pi_u(1)}\mathbf{R}_{uv}$ to be the minimum gap for user $u$ between the best item and any other hott item. For each $i\in [r]$, we will call the user $u^{\star}(i) \triangleq \s{argmax}_{u \in \ca{T}^{(i)}} \Delta_u$ with the largest hott item gap to be the most opinionated user in the $i^{\s{th}}$ cluster. Now, define $\Delta_{\s{hott}} \triangleq \min_{i\in [r]}\Delta_{u^{\star}(i)}$ to be hott item gap minimized across the most opinionated users across each cluster.
\end{defn}

\begin{defn}[Minimum Reward gap]\label{gap:standard}
Define $\Delta \triangleq \min_{u \in [\s{M}]} \fl{R}_{u\pi_u(1)}-\fl{R}_{u\pi_u(2)}$ to be the smallest gap between the best item and second best item minimized across all users.
\end{defn}

In most settings, the hott item gap $\Delta_{\s{hott}}$ is significantly larger than $\Delta$ -- the hott item gap needs to be satisfied by only a \emph{single} user from every latent cluster. To make this clear, consider the illustrative reward matrix in equation \eqref{eq:example} in Sec. \ref{sec:difficult} - for this matrix, $\Delta_{\s{hott}}=1$ and $\Delta = 2\epsilon/3$ from definitions - the latter quantity can be made as small as desired by setting $\epsilon$ appropriately.
We are now ready to present our first main result. All proofs are deferred to Appendix \ref{app:proofs}.

\subsection{Main Results}

\begin{thm}\label{thm:first}
Consider the online matrix completion problem with $\s{M}$ users, $\s{N}$ items, $\s{T}$ rounds such that at round $t\in [\s{T}]$, we observe reward $\{\fl{P}_{u\rho_u(t)}^{(t)}\}_{u\in [\s{M}]}$ as in eq. (\eqref{eq:obs}) with $\s{S}=1$ (i.e only one item recommended per user). Let $\fl{R}\in \bb{R}^{\s{M}\times \s{N}}$ be the expected rank-$r$ reward matrix. Suppose  Assumptions \ref{assum:hott}, \ref{assum:simplicity}, \ref{assum:matrix2} are true. Then Alg.\ref{alg:full} (\textsc{PhasedClusterElim}) guarantees the regret $\s{Reg}$ (eq. \eqref{eq:general}) to be $\widetilde{O}\Big(\max\Big(1,\s{N}\s{M}^{-1}\Big)\Big(\Delta_{\s{hott}}^{-2}+\Delta^{-1}\Big)+1\Big)$.
% \begin{align}
%     \mathsf{Reg}  = \widetilde{O}\Big(\max\Big(1,\s{N}\s{M}^{-1}\Big)\Big(\Delta_{\s{hott}}^{-2}+\Delta^{-1}\Big)+1\Big) 
% \end{align}
\end{thm}

Importantly, the regret guarantee above depends on the gap parameters $\Delta,\Delta_{\s{hott}}$ - the dependence on number of rounds $\s{T}$ is poly-logarithmic and is hidden within $\widetilde{O}(\cdot)$ notation \footnote{Gap-dependent regret bounds with logarithmic dependence on $\s{T}$ are stronger than gap-free worst case regret bounds with fractional power dependence on $\s{T}$. (see for example \cite{lattimore2020bandit}, Theorem 7.1). Usually, we can convert gap-dependent regret bounds to worst-case regret bounds by setting the gap parameter appropriately as function of  $\s{N},\s{T}$.
} 
For readability, we have also hidden the polynomial dependence on $r,\sigma,\kappa$ that are assumed to be constants (\textbf{A2}). 

The regret guarantee in Theorem \ref{thm:first} has an optimal linear dependence on $\s{N}\s{M}^{-1},\Delta^{-1}$ and a quadratic dependence on $\Delta_{\s{hott}}^{-1}$; it is significantly better than the $\Delta^{-2}$ dependence achieved by the greedy algorithm proposed in \cite{jain2022online}. 
Again, we stress that the hott item gap $\Delta_{\s{hott}}$ is  significantly larger than the minimum reward gap $\Delta$ from definition. While the quadratic dependence on $\Delta_{\s{hott}}^{-1}$ is unwanted, $\Delta_{\s{hott}}$ being large should make its contribution to the regret small in practice. 

\begin{rmk}
    The optimality of the dependence on $\s{N},\s{M},\Delta$ can be observed from a simple example. Suppose all users are equivalent - hence, rows of $\fl{R}$ are identical and in a single round, users come in a sequential fashion too. This is strictly easier and  equivalent to a multi-armed bandit (MAB) problem with $\s{N}$ items and $\s{MN}$ rounds up to normalization. Standard MAB literature  imply that asymptotic lower bound for our setting is $\widetilde{O}(\s{N}\s{M}^{-1}\Delta^{-1})$ - see\cite{lattimore2020bandit}, Chapter 16.   
\end{rmk}

\begin{rmk}
     Note that our guarantees in Theorem \ref{thm:first} can also be extended easily to the setting where the number of hott vectors is $t>r$. In the extreme case, when all the item embedding vectors lie on the surface of a sphere,  the number of hott vectors is same as total number of items.
     However, if number of hott vectors is smaller, then we have non-trivial regret guarantees. Hence, we can also think of our result as a more fine-grained parameterization in terms of the number of hott vectors which can be small in several practical applications.
\end{rmk}

\noindent \textbf{Setting 2  (Simplified Regret):} Next, we consider a related setting with $\s{S}=r$ i.e. at each round $t\in [\s{T}]$, for each user $u\in [\s{M}]$, $r$ distinct items $\rho_u(t,1),\dots,\rho_u(t,r)\in [\s{N}]$ are recommended and we observe the corresponding noisy rewards.
Now, instead of Assumption \ref{assum:matrix2}, we make the following milder assumption (similar to \cite{kveton2017stochastic}) that the latent embeddings of users and items are non-negative:
\begin{assumption}\label{assum:non-negative}
 The user (item) embedding matrix $\fl{U}$ ($\fl{V}$) belongs to the set $\ca{S}_{\s{M},d}$ ($\ca{S}_{\s{N},d}$). 
\end{assumption}
In this setting, we consider the simplified regret:
\begin{align}\label{eq:simple}
\s{Reg}_{\s{Simple}} \triangleq \frac{1}{\s{M}} \Big(\sum_{u\in[\s{M}]}\s{T}\mathbf{R}_{u\pi_u(1)}- \sum_{\substack{t\in[\s{T}] \\ u\in[\s{M}]}}\max_{j\in [\s{r}]}\mathbf{R}_{u\rho_u(t,j)}\Big).
\end{align}
Note that at each round $t$, the simplified regret for a user $u$ is small if one of the recommended items $\{\rho_{u}(t,j)\}_{j\in [r]}$ is close to the best item $\pi_u(1)$. In general the simplified notion of regret (eq. \eqref{eq:simple}) is easier to optimize than the general regret (eq. \eqref{eq:general}) - more precisely, if there exists an algorithm guaranteeing the regret $\s{Reg}=O(h)$ for some $h$ that is a function of $\s{M},\s{N},\s{T}$ by recommending items with $\s{S}=1$, then the same algorithm can also guarantee $\s{Reg}_{\s{Simple}}=O(h)$ with $\s{S}=r$. However, as shown below, our theoretical guarantees for $\s{Reg}_{\s{Simple}}$ hold under significantly milder assumptions than that for $\s{Reg}$.
Due to Assumption \ref{assum:hott}, we can conclude that the squared determinant $\s{det}^2(\fl{V}_{\ca{A}})$ of the item embedding matrix restricted to the hott items is larger than $\s{det}^2(\fl{V}_{\ca{A}'})$ for any $\ca{A}'\neq \ca{A}, \left|\ca{A}'\right|=r$ (see \cite{kveton2017stochastic}). Thus, we can define the following:

\begin{defn}[Determinant gap]\label{defn:det}
    We define the determinant sub-optimality gap $\Delta_{\s{det}}$ as follows: $\Delta_{\s{det}} =\s{det}^2(\fl{V}_{\ca{A}})-\max_{\ca{A}'\neq \ca{A}, \left|\ca{A}'\right|=r} \s{det}^2(\fl{V}_{\ca{A}'})$.
%\begin{align}
%     \Delta_{\s{det}} =\s{det}^2(\fl{V}_{\ca{A}})-\max_{\ca{A}'\neq \ca{A}, \left|\ca{A}'\right|=r} \s{det}^2(\fl{V}_{\ca{A}'}).   
%\end{align}
\end{defn}
We now state our second main result:
\begin{thm}\label{thm:second}
Consider the online matrix completion problem with $\s{M}$ users, $\s{N}$ items $\s{T}$ rounds such that at round $t\in [\s{T}]$, we observe noisy reward $\{\fl{P}_{u\rho_u(t,j)}^{(t)}\}_{u\in [\s{M}],j\in [r]}$ as in eq. (\eqref{eq:obs}) with $\s{S}=r$, noise variance proxy $\sigma^2>0$. Let $\fl{R}\in \bb{R}^{\s{M}\times \s{N}}$ be the expected rank-$r$ reward matrix. Suppose  Assumptions \ref{assum:hott}, \ref{assum:non-negative} are true. Then Alg.\ref{algo:simple} (\textsc{DeterminantElim}) guarantees regret $\s{Reg}_{\s{Simple}}$ (eq. \eqref{eq:simple}) to be $\widetilde{O}\Big(\frac{\sigma^2\s{N}}{c_{\s{avg}}c_{\s{max}}\s{M}^{1/r}\Delta_{\s{det}}}+1\Big)$
% \begin{align}
%     \mathsf{Reg}_{\s{Simple}}  = . 
% \end{align}
where $c_{\s{avg}}\triangleq {\s{M} \choose r}^{-1} \sum_{\ca{I}\subset [\s{M}]\mid \left|\ca{I}\right|=r} \s{det}^2(\fl{U}_{\ca{I}})$ and $c_{\max}\triangleq \s{det}^2(\fl{V}_{\ca{A}})$.
\end{thm}

As in Theorem \ref{thm:first}, the poly-logarithmic dependence on $\s{T}$ and the poly dependence on $r,\sigma,\kappa$ (assumed to be constants) are subsumed within $\widetilde{O}(\cdot)$.
The regret guarantee in Theorem \ref{thm:second} has a linear dependence on the inverse of determinant gap $\Delta^{-1}_{\s{det}}$. 
For rank-$r>1$, we do suffer a cost in form of the factor $\s{M}^{-1/r}$ in the regret. 
For $r=1$ (rank-$1$), the two proposed notions of regret are equivalent and the guarantee in Theorem \ref{thm:second} recovers the guarantees in \cite{jain2022online} for rank-$1$ setting. 

\begin{rmk} \label{rmk:run}
Note that our proposed algorithms are computationally efficient with run-times having polynomial dependence on $\s{M},\s{N}$. Their run-times do have an exponential dependence on the rank-$r$ and hence, the proposed algorithms are computationally feasible for small values of rank $r$. Our poly dependence on $\s{M},\s{N}$ stems from the use of low rank matrix completion algorithms as a sub-module (see Step 17 in Algorithm \ref{alg:full}) - the convex relaxation based approach used crucially for strong theoretical guarantees is computationally less efficient for matrix completion. However, in practice, a number of highly efficient techniques have been proposed for optimizing matrix completion - for example, see \cite{recht2013parallel,teflioudi2012distributed}. Substituting such estimators for invoking matrix completion can allow us to scale significantly.
\end{rmk}

\section{Algorithm \textsc{PhasedClusterElim} and Analysis of $\s{Reg}$  (Theorem\ref{thm:first})}\label{sec:difficult}

\begin{algorithm*}
\small
\caption{\textsc{PhasedClusterElim} (Iterative Multi-Label User Classification and Joint Item Elimination)}
\label{alg:full}
%\textbf{Input}: 
\begin{algorithmic}[1] 
\STATE Initialize $G$ to be a fully connected graph of $\s{M}$ users. Set $\cT^0_u = [\s{N}]$ for all users $u \in [\s{M}]$, phase number $\ell = 0$, $k_0 = 1$ and $c_0 = 3k_0$. 

\textcolor{blue}{//\texttt{ Stage 1: Identifying a set of \textit{opinionated users} via independent set}}

\WHILE{$\texttt{CheckIndependentSet}(G)$ returns False}
    %\STATE $\tilde {\fl{R}}^{\ell+1} \leftarrow \texttt{ESTIMATE}(\s{[M]},\s{[N]}, k_\ell, p, \lambda,s,\cT_u^{\ell}|_{u=1}^{\s{M}})$
    \STATE $\tilde {\fl{R}}^{\ell+1} \leftarrow \mathcal{O}(\s{[M]},\s{[N]}, k_\ell)$ (Invoke Algorithm \ref{alg:estimate}).
    \STATE Update the item set for each user as $\cT^{\ell+1}_u = \{ x \in \cT^{\ell}_u : \tilde{\fl{R}}^{\ell+1}_{u,\tilde \pi_u(1)} - \tilde{\fl{R}}^{\ell+1}_{u,x} \le c_\ell \}$
    \STATE Update the graph G such that two users $u$ and $v$ are connected iff $\cT_{u}^{\ell+1} \cap \cT_{v}^{\ell+1} \neq \Phi$; Set $\ell \leftarrow \ell +1$, $k_\ell \leftarrow k_{\ell-1}/2$ and $c_\ell = 3k_\ell$
\ENDWHILE
\STATE Let $u_1,\ldots,u_r$ be the users returned by the last call to $\texttt{CheckIndependentSet}(G)$

\STATE Set $k_\ell \leftarrow k_{\ell-1}/10$ and $c_{\ell}=3k_{\ell}$.

\STATE $\tilde{\fl{R}}^{\ell+1} \leftarrow \texttt{ESTIMATE\_SIMPLE}(\s{[M]},\s{[N]}, k_{\ell})$ 

\textcolor{blue}{//\texttt{ Stage 2: Joint Elimination of items via expansion and intersection}}

\WHILE{ time horizon of the game is not reached}
    \STATE Update the item set for each user as $\cT^{\ell+1}_u = \{ x \in \cT^{\ell}_u : \tilde{\fl{R}}^{\ell+1}_{u,\tilde \pi_u(1)} - \tilde{\fl{R}}^{\ell+1}_{u,x} \le c_\ell \}$
    \STATE Update the graph G such that two users $u$ and $v$ are connected iff $\cT_{u}^{\ell+1} \cap \cT_{v}^{\ell+1} \neq \Phi$; Set $\ell \leftarrow \ell +1$; $k_\ell \leftarrow k_{\ell-1}/2$ and $c_\ell = 3 k_\ell$

\STATE Update candidate item sets by calling $\texttt{ExpanditemSets}(G,(u_1,\ldots,u_r),c_{\ell-1},\tilde{\fl{R}}^\ell)$

\STATE Let $\cT^{\ell}_{C_i}$ be the common good item set for all users that are connected to user $u_i$

\FOR{ each cluster $i \in [r]$}
    \STATE $\ca{U} \leftarrow \{u: \cT_u^\ell \cap \cT_{C_i}^\ell \neq \Phi \}$ \textcolor{blue}{//\texttt{ Recommend items to users in  $\ca{U}$ by calling $\mathcal{O}$ as below}}
    
    \STATE $\tilde{\fl{R}}^{\ell+1}_{\ca{U},\cT_{C_i}^{\ell}} \leftarrow \mathcal{O}(\ca{U},\cT_{C_i}^\ell,k_{\ell})$ (Invoke Algorithm \ref{alg:estimate})
    
    \textcolor{blue}{//\texttt{$\mathcal{O}$ collects data for $O(k_{\ell}^{-2} \max(1,\s{N}\s{M}^{-1}))$ rounds and estimates matrix $\fl{R}_{\ca{U},\cT_{C_i}^{\ell}}$}. In practice the data collection mechanism can be random sampling from the desired sub-matrix (Algorithm \ref{alg:estimate})}
\ENDFOR

\ENDWHILE

\end{algorithmic}
\end{algorithm*}

Due to space restrictions, certain  modules of algorithm are delegated to Appendix \ref{app:proofs}. Below, we provide
an overview of the main components of Alg. \textsc{PhasedClusterElim} (see Algorithm \ref{alg:full}) and their analysis. The algorithm is divided into two distinct stages, each of which runs in phases of exponentially increasing length. Note that we ensure every user $u\in[\s{M}]$ is recommended one item at each round while invoking $\mathcal{O}$ (offline low rank matrix completion module).

\noindent \textbf{Stage 1 (Identifying a set of \textit{opinionated users} via independent set):} Our first key step is to identify a set of $r$ users such that any pair among them clearly do not share the same best item. We find these set of $r$ users in the following way:  for the first few rounds, we explore  i.e. for every user, we recommend randomly sampled items (See Remark \ref{rmk:explore}) via the call to $\mathcal{O}$ and store the corresponding noisy observations (Line 3). The rows and columns of the reward matrix is given by the first and second arguments of $\mathcal{O}$ respectively (which are set equal to $\s{[M]}$ and $\s{[N]}$ in Line 3). At this point, we have a partially observed (noisy) reward matrix. Due to its low rank, we can invoke guarantees of well-known low rank matrix completion estimators that allow us to compute an estimate of the entire reward matrix from the few observed entries with an error guarantee (equal to $k_\ell$, see Line 3) on each entry. The task of matrix completion is also carried out while invoking $\mathcal{O}$ (offline low rank matrix completion). Based on the estimated rewards, we maintain a candidate item set for each user (denoted by $\cT_u^{\ell+1}$ at phase $\ell$, Line 4). This item set is used to construct a graph with users as nodes such that two users are connected iff their candidate item sets have a non-empty intersection. The idea of the graph is to cluster similar users (in terms of their latent embeddings) together. As the phase number increases, we get finer estimates of the reward matrix (the argument $k_\ell$ in the call to $\mathcal{O}$ controls the estimation accuracy level) . Consequently the size of the candidate item sets become small and many of the nodes in the graph get disconnected, thereby surfacing out similar users. We prove that after $\log(1/\Delta_{\text{hott}})$ phases, an independent set of size $r$ arises in the graph, where $\Delta_{\text{hott}}$ is as in Definition \ref{defn:weird_gap}. Since the independent set has size $r$, there must exist at-least $r$ users that are mutually disconnected in the graph. We call such users \emph{opinionated} users since it can be shown that their hott-item gap is at-least $\Delta_{\text{hott}}$ (see Defn. \ref{defn:weird_gap}). We prove that at every phase, the best item for \emph{any} user belongs to their candidate item sets with high probability. Consequently, the set of $r$ opinionated users must have distinct best items.

Lines 8-9 are necessary for technical reasons to ensure that if an item is not in a user's candidate item set, its reward for that user is sufficiently lower than the reward of the best item.

\begin{rmk}\label{rmk:explore}
    The recommendation of items to users is being done inside the \texttt{ESTIMATE\_SIMPLE} function (Algorithm \ref{alg:estimate}) in L17 of Algorithm \textsc{PhasedClusterElim}. 
    On invoking \texttt{ESTIMATE\_SIMPLE} function , the algorithm recommends randomly sampled items from the input subset of items in the second argument of the function to the input subset of users in the first argument (the third argument determine the number of rounds for which to recommend). To invoke the theoretical guarantees of offline low rank matrix completion \cite{chen2019noisy}, the recommendation strategy needs to be slightly modified (see Sec. 3 in \cite{jain2022online}) to account for our setting. Furthermore, the theoretical estimator used is the convex relaxation approach wherein we minimize the MSE with a nuclear norm regularizer - for details, see the more sophisticated Algorithm \ref{algo:estimate} (equivalent to $\mathcal{O}$) in Appendix, for which, sharp theoretical guarantees can be invoked.
\end{rmk}

After identifying $r$ opinionated users, our strategy is to utilize the opinionated users as labels for $r$ groups i.e. if we say that a user $u\in [\s{M}]$ has similarity with a subset of opinionated users, then the best item of $u$ must be among the best items (hott items) of the aforementioned subset of opinionated users. The correct label for each user is determined by the opinionated user who shares the same best item. If a user is connected to multiple opinionated users, we assign multiple labels to capture the uncertainty in their best item. At most, we assign $r$ labels per user, but we aim to reduce the number of labels as we gather more information. This brings us to the second stage of our algorithm:

\noindent \textbf{Stage 2 (Joint Elimination of Items via Expansion and Intersection):} Our strategy for refining the labels is to explore carefully while exploiting the user-user and item-item similarities learned from the previous phase. User-user similarity is leveraged by forming groups of users with the same assigned label (for non-opinionated user, this means that the label under consideration is present in its multi-label set). There can be $r$ such (possibly non-disjoint) groups. For concreteness, we restrict our attention to a group corresponding to label $u_i$ hereof, where $u_i$ is an opinionated user with $\ca{A}_i$ as its best hott-item (see Assumption \ref{assum:hott}). Thus all users in this group have connection to the user $u_i$.  Such a group corresponding to a label $u_i$ is stored in $\ca{U}$ in Line 16 of Algorithm \ref{alg:full}. Due to our graph design, the set of all users whose best item is $\ca{A}_i$ must be a \emph{subset} of $\ca{U}$. To exploit the item-item similarity we can take the intersection of the candidate item sets of the users in a group (see Line 14). This set of items must be high rewarding for all users within that group. However there can be multiple users (for eg. the user $u_i$) in $\ca{U}$ with best item being $\ca{A}_i$. So the intersection can potentially eliminate the best item $\ca{A}_i$. This can happen because of the subset condition noted before, there can also be users in $\ca{U}$ that do not have $\ca{A}_i$ in its candidate item set. Specifically, a non-opinionated user $u$ with its best item not equal to $\ca{A}_i$ can be connected to $u_i$ (or equivalently has a label $u_i$) due to some $x \in \cT_{u_i}^\ell \cap \cT_{u}^\ell$ with $x \neq \ca{A}_i$. In order to resolve this issue, we expand (Line 13) the candidate item set for each non-opinionated user slightly so that provably: 1) the expanded set will contain the true best item (hott item) of the user $u$ and the item $\ca{A}_i$ and 2) the set of labels assigned to user $u$ is same as that before the expansion. As mentioned before, best item for \emph{any} user belongs to their candidate item sets with high probability. Subsequently, taking intersection of the candidate item sets of all users with label $u_i$ assigned to them (denoted by $\cT_{C_i}^\ell$ in Line 14) will ensure that the item $\ca{A}_i$ will not be eliminated from $\cT_{C_i}^\ell$. So we can randomly explore within the sub-matrix defined by the rows $\ca{U}$ and columns $\cT_{C_i}^\ell$ via the call to $\mathcal{O}$ in Line 17 with a finer estimation error $k_\ell$. This ensures that the regret due to the exploration is low enough and at the same time the best item for each user is not eliminated when updating the candidate item sets via the refined estimates of reward (sub) matrix in Line 11. The process is repeated across phases, increasing confidence in user similarity and rewarding items. This is reflected in the algorithm by a decrease in multi-labelled users as well as a shrinkage of candidate item set sizes.

\begin{algorithm*}

\caption{\textsc{Estimate\_Simple} (Function for estimating reward sub-matrix restricted to a subset of users and items)}
\label{alg:estimate}
\begin{algorithmic}[1] 
\STATE \textbf{Input:} Set of users $\ca{U}\subseteq [\s{M}]$, $\ca{V}\subseteq [\s{N}]$, desired entry-wise error $\eta$. 
\FOR{round index $1,2,\dots,O(\eta^{-2} \max(1,\s{N}\s{M}^{-1}))$}
\STATE For each user $u\in \ca{U}$, recommend to $u$ an item $v\in\ca{V}$ sampled randomly. Store the observed feedback.
\ENDFOR

\STATE Use offline Matrix Completion with observed feedback as input data and return an estimate of the matrix $\fl{R}_{\ca{U},\ca{V}}$.

\textcolor{blue}{//\texttt{There exist several standard algorithms for offline Low Rank Matrix Completion with partial noisy observations. Two well-known approaches are 1) Minimizing MSE with nuclear norm regularizer 2) Alternating Minimization (see \cite{jain2017non} and references therein).}}

\end{algorithmic}
\end{algorithm*}

\noindent \textbf{Regret Analysis:} Stage 1 of our algorithm is greedy i.e. we incur worst-case regret until we have found a set of $r$ \textit{opinionated users}. However, we only need to find one \textit{opinionated user} per hott item and therefore the regret for Stage 1 scales as $\widetilde{O}(\max(1,\s{N}\s{M}^{-1})\Delta_{\s{hott}}^{-2})$. For the subsequent stages, we can show that in an phase with length 
$\widetilde{O}(\max(1,\s{N}\s{M}^{-1})\Delta^{-2})$
(needed for getting the estimates of the relevant sub-matrices up to error $O(\Delta)$), the regret incurred per round is at most $\Delta \ge \Delta$. Combining the above and the fact that the total number of phases is at most $O(\log \s{T})$, we obtain the regret described in Theorem \ref{thm:first}.

To ensure that expanding the item sets for non-opinionated users does not increase regret, we work with latent embeddings of users and items. Consider the scenario where a user $u$ is connected to opinionated users $u_i$ and $u_j$ with their respective best items $\ca{A}_i$ and $\ca{A}_j$. Let $\tilde \cT_v^\ell$ denote the candidate item sets for any user $v$ before Line 13. Suppose $\ca{A}_i$ is not in $\tilde \cT_u^\ell$ but there exists $x$ in $\tilde \cT_u^\ell \cap \tilde \cT{u_i}^\ell$ due to their connection. The call to \texttt{ExpanditemSets} ensures that both $\ca{A}_i$ and $\ca{A}_j$ are in $\cT_u^{\ell}$. We aim to demonstrate control over $\fl{R}_{u,\ca{A}_j} - \fl{R}_{u,\ca{A}_i}$ since $\ca{A}_i$ will be recommended to user $u$ during the $\mathcal{O}$ call. Let $e$ and $e'$ denote the latent embedding maps of users and items respectively. We begin with the regret decomposition
\begin{align}
    &\fl{R}_{u,\ca{A}_j} - \fl{R}_{u,\ca{A}_i}
    = e(u)^Te'(\ca{A}_j-\ca{A}_i)
    = e(u)^Te'(\ca{A}_j-x) \\
    &+ e(u_i)^Te'(\ca{A}_i-x)
    + e(u_i+u)^Te'(x-\ca{A}_i)\\
    &\le \left(\fl{R}_{u,\ca{A}_j} - \fl{R}_{u,x} \right) + \left(\fl{R}_{u_i,\ca{A}_i} - \fl{R}_{u_i,x}\right).
\end{align}
We show that for a hypothetical user defined by the embedding $e(u_i+u)$, its best item must be $\ca{A}_i$, due to which the last inequality follows. Since $x \in \tilde \cT_u^\ell \cap \tilde \cT_{u_i}^\ell$, even before the expansion step, the last two terms can also be controlled (at the level $O(k_\ell))$. 
%In fact, $\fl{R}_{u,\ca{A}_j} - \fl{R}_{u,\ca{A}_i}$ being controlled at the level $O(k_\ell)$ is what 
This also guarantees that $\ca{A}_i$ will belong to the modified item set of user $u$ after a slight expansion of the item set.

\noindent \textbf{Illustrative example on why simple item elimination fails: }
Now, we provide a working example to motivate why a careful design of algorithm is required to attain optimal regret. Suppose for simplicity we know the value of $\Delta_{\text{hott}} = \Theta(1)$. Then by running Stage 1 $\Theta(1)$ number of times, we can guarantee that the suboptimality of any candidate item for any user is at-most $\Delta_{\text{hott}}/4$ (for eg. see Lemma \ref{lem:subopt}). Consider the following claim which we eventually show to be false.

\begin{prop} \label{prop:false}
Suppose we run Stage 1 of Alg. \textsc{PhasedClusterElim} so that the sub-optimality of each candidate item is at-most $\Delta_{\text{hott}}/4$. Then if two users $u$ and $v$ do not share the same optimal item, their items sets maintained at the current round will be disjoint.
\end{prop}

Before disproving the proposition, we note that if it was true, then it will imply a simple item elimination strategy as follows.  We can run Stage 1 for a constant number of times. Then we cluster together all users who have non-empty intersection of their candidate items and play items that are present in the intersection of their candidate item sets. Afterwards we can refine the intersection of their candidate item sets via matrix completion. This will also guarantee that the regret will remain controlled since the intersection of the item sets do not eliminate the best item of users within a cluster. Unfortunately the proposition above can be false as shown by the counter-example below.

Consider a ground truth reward matrix as follows:
\begin{align}\label{eq:example}
    R = 
    \begin{bmatrix}
    1 & 0 & 1/3\\
    0 & 1 & 2/3\\
    p & p-\epsilon & p-2\epsilon/3
    \end{bmatrix}
\end{align}
with the corresponding user and item embeddings denoted by $U = \begin{psmallmatrix} 1 & 0\\ 0 & 1 \\ p & p-\epsilon \end{psmallmatrix}$ and $V = \begin{psmallmatrix}1 & 0\\ 0 & 1 \\ 1/3 & 2/3 \end{psmallmatrix}$ where $\epsilon < 1/4$ and $1/4 < p < 1$. Let user corresponding to row $i$ of reward matrix be denoted by $u_i$ and item corresponding to column $i$ be denoted by $v_i$. Here the hott-topics are items $v_1$ and $v_2$ and the hott-topic gap  $\Delta_{\text{hott}} = 1$ and $\Delta=2\epsilon/3$. Further $u_1$ and $u_2$ are the opinionated users as defined in Proposition \ref{ass:opinion}. By Lemma \ref{lem:subopt}, we attain a sub-optimality of the candidate item set of $1/4$ by running Stage 1 until $k_\ell = 1/20$. Then the candidate item for user $u_1$ is $\{ v_1\}$, $u_2$ is $\{ v_2 \}$ and $u_3$ may contain $\{v_1,v_2,v_3 \}$. Thus Proposition \ref{prop:false} only holds true for opinionated users and can be false for non-opinionated user $u_3$ since the item sets for $u_3$ and $u_1$ intersect although their optimal items are different.

\textbf{Takeaway message.} The primary challenge in our algorithm lies in how to recommend items to less opinionated users while still exploiting their mutual similarities. This required us to do careful identification of relevant sub-matrices that are forwarded to downstream estimation (function ESTIMATE Alg. \ref{alg:estimate}) - here, noisy entries of the sub-matrix are observed to collect data for sub-matrix completion task.

\section{Algorithm \textsc{DeterminantElim} and Guarantees on $\s{Reg}_{\s{Simple}}$ (Theorem\ref{thm:second})} \label{sec:simple}

\begin{algorithm}[t]

\caption{\textsc{DeterminantElim} 
  (Collaborative Filtering via Phased Elimination of $r$-columns)
\label{algo:simple}}
\begin{algorithmic}[1]

\REQUIRE Set of users $[\s{M}]$, set of items $\s{N}$, rounds $\s{T}$, noise variance $\sigma^2>0$.

\STATE Initialize $\ca{B}^{(1)}=\Pi_r(\s{N})$ and $d_1=O(\s{N}\s{M}^{-1/r}C(r))$ where $C(r)=2^{r+1}r^{1+r/2}$.

\FOR{$\ell=1,2,\dots$}

\STATE For each user $u\in [\s{M}]$, sample uniformly at random $d_{\ell}$ items from $\cup_{\ca{J}\in \ca{B}^{(\ell)}}\ca{J}$. Store the sampled $d_{\ell}$ items in the set $\ca{T}_u^{(\ell)}$.
$\quad$
\textcolor{blue}{//\texttt{ Beginning of a phase}} \label{line:sample}

\STATE For the next $2\sigma^2d_{\ell}$ rounds, for each user $u\in [\s{M}]$, for each item $z$ in $\ca{T}_u^{(\ell)}$, recommend some $\ca{J}\in \ca{B}^{(\ell)}$ $2\sigma^2$ times such that $z\in \ca{J}$. \textcolor{blue}{//\texttt{ Recommending items to users}} \label{line:recommend}

\STATE Initialize $\widetilde{\fl{R}}^{(\ell,1)},\widetilde{\fl{R}}_{uz}^{(\ell,2)}\in \f{0}^{\s{M}\times \s{N}}$ to be all zero matrices.
For each user $u\in [\s{M}]$, for each item $z\in \ca{T}_u^{(\ell)}$, we store $\widetilde{\fl{R}}_{uz}^{(\ell,1)}$ ($\widetilde{\fl{R}}_{uz}^{(\ell,2)}$ ) to be the mean of the first $\sigma^2$ (remaining $\sigma^2$) observations corresponding to the recommendation of item $z$ to user $u$. \label{line:store}

\STATE For each $\ca{J}\in \ca{B}^{(\ell)}$, compute $\widetilde{\mu}_{\ca{J}}$ as in equation \eqref{eq:estimate}. \textcolor{blue}{//\texttt{ Estimate avg. determinant}} \label{line:estimate}

\STATE Compute $ \ca{B}^{(\ell+1)}$ to be $\{\ca{J}\in \ca{B}^{(\ell)}\mid \widetilde{\mu}_{\ca{J}} \ge \max_{\ca{J}'\in \ca{B}^{(\ell})}\widetilde{\mu}_{\ca{J}'}-2\cdot2^{-\ell} \}$. \label{line:eliminate}

\STATE Update $d_{\ell+1}\leftarrow 4d_{\ell}$. $\quad$ \textcolor{blue}{//\texttt{ End of a phase}}

\ENDFOR

\end{algorithmic}
\end{algorithm}

 % We first start with an overview of Algorithm \textsc{DeterminantElim} that is simpler to state and analyze.

In this section, we consider the setting when at each round $t\in [\s{T}]$, for each user $u\in [\s{M}]$, $r$ distinct items $\rho_u(t,1),\rho_u(t,2),\dots,\rho_u(t,r)\in [\s{N}]$ are recommended with the goal of minimizing the simplified regret $\s{Reg}_{\s{Simple}}$ defined in equation \eqref{eq:simple}. In this setting, we make benign assumptions on the reward matrix $\fl{R}=\fl{U}\fl{V}^{\s{T}}$ namely Assumptions \ref{assum:hott} and \ref{assum:non-negative}. 
%However, Assumption \ref{assum:hott}  leads to an interesting property of the reward matrix $\fl{R}$, also observed in \cite{kveton2017stochastic}. 
Note that for any $r$-row $\ca{I}$ and $r$-column $\ca{C}$, we must have that $\s{det}(\fl{R}_{\ca{I},\ca{C}})=\s{det}(\fl{U}_{\ca{I}})\s{det}(\fl{V}_{\ca{C}})$ where $\fl{R}_{\ca{I},\ca{C}}$ is the sub-matrix of $\fl{R}$ restricted to rows in $\ca{I}$ and columns in $\ca{C}$. This implies that $\s{det}(\fl{R}_{\ca{I},\ca{C}})$ is a scaled version of $\s{det}(\fl{V}_{\ca{C}})$. Recall that $\ca{A}\subseteq [\s{N}]$ corresponds to the indices of the $r$ hott columns. Notice that determinant of a matrix corresponds to the signed volume  determined by the matrix rows and zero vector.
 The hott items assumption ensures that the volume of $\fl{V}$ restricted to indices in $\ca{A}$ is largest - hence, $\s{det}^2(\fl{V}_{\ca{A}})>\s{det}^2(\fl{V}_{\ca{A'}})$ for any $r$-column $\ca{A}'\neq \ca{A}$ (see for eg. \cite{kveton2017stochastic}) and therefore, we must have that $\s{det}^2(\fl{R}_{\ca{I},\ca{A}})>\s{det}^2(\fl{R}_{\ca{I},\ca{A}'})$ for any $r$-row $\ca{I}$. 
 This property is exploited crucially by our proposed algorithm \textsc{DeterminantElim} to gradually eliminate sub-optimal $r$-columns 
 for this setting and converge to the hott item indices $\ca{A}$. Moreover, for any $r$-column $\ca{A}'\neq \ca{A}$, we define sub-optimality gap of the $r$-column $\ca{A}'$ to be $\s{det}^2(\fl{V}_{\ca{A}})-\s{det}^2(\fl{V}_{\ca{A}'})$.

Algorithm \textsc{DeterminantElim} is initialized with the set of all possible $r$-columns. The algorithm runs in phases of exponentially increasing length and gradually eliminates $r$-columns that are sub-optimal in every phase.  
At the beginning of  $\ell^{\s{th}}$ phase, suppose $\ca{B}^{(\ell)}$ is the set of surviving $r$-columns. Consider the following quantity which will be central to our algorithm design:
\begin{align}\label{eq:weird}
    \mu_{\ca{J}}\triangleq \frac{1}{{\s{M} \choose r}} \sum_{\ca{I}\subset [\s{M}]\mid \left|\ca{I}\right|=r} \s{det}^2(\fl{R}_{\ca{I},\ca{J}}) %\le \frac{1}{{\s{M} \choose r}} \sum_{\ca{I}\subset [\s{M}]\mid \left|\ca{I}\right|=r} \s{det}^2(\fl{R}_{\ca{I},\ca{A}}) = \mu_{\ca{A}}.
\end{align}
For a fixed set of $r$ items $\ca{J}$, $\mu_{\ca{J}}$ is the average of determinant of all 
${\s{M} \choose r}$ $r\times r$ sub-matrices of $\fl{R}$ formed by the $r$-column $\ca{J}$ and each of the ${\s{M} \choose r}$
$r$-rows of $\fl{R}$.
Clearly, we have $\mu_{\ca{J}} \le \mu_{\ca{A}}$. Hence our next goal will be to estimate $\mu_{\ca{J}}$ for each $r$-column $\ca{J} \in \ca{B}^{(\ell)}$ and eliminate $r$-columns in $\ca{B}^{(\ell)}$ that have low estimated values. By doing so we hope to gradually identify the $r$-column defined by $\ca{A}$.

Denote the set of columns that belongs to some surviving $r$-column in $\ca{B}^{(\ell)}$ to be $\ca{Y}\triangleq \cup_{\ca{J}\in \ca{B}^{(\ell)}} \ca{J}$. To compute an estimate $\widetilde{\mu}_{\ca{J}}$ of $\mu_{\ca{J}}$ for all $\ca{J}\in \ca{B}^{(\ell)}$, for every user, a set $\ca{T}_u^{(\ell)}\subseteq \ca{Y}$ of $d_{\ell}$ distinct items are uniformly sampled at random from $\ca{Y}$. Then, throughout the phase of length $2\sigma^2 d_{\ell}$, we recommend items in $\ca{T}_u^{(\ell)}$ to user $u$ and construct two independent estimates $\widetilde{\fl{R}}^{(\ell,1)},\widetilde{\fl{R}}^{(\ell,2)}$ of the reward matrix $\fl{R}$ (Lines 3-5). 
%  we set the total number of rounds in phase $\ell$ to be .
% Recall that at each round, we recommend $r$ distinct items to every user. Now, in the ensuing $2d_{\ell}$ rounds, for each user $u\in [\s{M}]$, we do the following - we take each item $z$ in their sampled set of items $\ca{T}_u^{(\ell)}$ and we recommend some surviving $\ca{J}\in \ca{B}^{(\ell)}$ twice such that $z\in \ca{J}$ (Line \ref{line:recommend} in Alg. \ref{algo:simple}). In this manner, we  obtain the noisy observations for each user $u\in [\s{M}]$ corresponding to each item in $\ca{T}_u^{(\ell)}$ at least twice. 3) Third, we initialize two matrices $\widetilde{\fl{R}}^{(\ell,1)},\widetilde{\fl{R}}^{(\ell,2)}\in \bb{R}^{\s{M}\times \s{N}}$ to be all zero matrices. For each user $u\in [\s{M}]$ and for each item $z\in \ca{T}_u^{(\ell)}$, we store the average of the first noisy observation obtained on recommending item $z$ to user $u$ in $\widetilde{\fl{R}}^{(\ell,1)}_{uz}$. Similarly, we store the average of another $\sigma^2$ noisy observations obtained on recommending item $z$ to user $u$ in $\widetilde{\fl{R}}^{(\ell,2)}_{uz}$ (Line \ref{line:store} in Alg. \ref{algo:simple}). 
Finally, for a fixed set of $r$ surviving items $\ca{J}\in \ca{B}^{(\ell)}$, let $\ca{H}_{\ca{J}}\subseteq [\s{M}]$ be the set of users for which $\ca{J}\subset \ca{T}_u^{(\ell)}$ i.e., the items in $\ca{J}$ are present in the sampled items $\ca{T}_u^{(\ell)}$ for every user $u\in\ca{H}_{\ca{J}}$, 
%- for these users, all items in $\ca{J}$ were recommended at least $2\sigma^2$ times. 
We compute estimate $\widetilde{\mu}_{\ca{J}}$ of $\mu_{\ca{J}}$ as  (Line \ref{line:estimate}):
\begin{align}\label{eq:estimate}
 \widetilde{\mu}_{\ca{J}}\triangleq \frac{1}{{n_{\ca{J}} \choose r}} \sum_{\ca{I}\subset \ca{H}_{\ca{J}}\mid \left|\ca{I}\right|=r} \s{det}(\widetilde{\fl{R}}^{(\ell,1)}_{\ca{I},\ca{J}})\cdot \s{det}(\widetilde{\fl{R}}^{(\ell,2)}_{\ca{I},\ca{J}})
\end{align}
We prove that in phase $\ell$, it suffices to choose the scaled length of the phase $d_{\ell}=O(\s{N}\s{M}^{-1/r}C(r)2^{2\ell})$ where $C(r)=2^{r+1}r^{1+r/2}$. Based on our estimates, we can eliminate $r$-columns from $\ca{B}^{(\ell)}$ as shown in Line \ref{line:eliminate}.

Note that the quality of the estimate $\widetilde{\mu}_{\ca{J}}$ improves as $\left|\ca{H}_{\ca{J}}\right|$increases. One of the key insights is to show that $\left|\ca{H}_{\ca{J}}\right|\ge n_{\ell}$ for a fixed $n_{\ell}$ if $d_{\ell}$ satisfies the condition in the following lemma.

\begin{lemma}\label{lem:sample_carefully}
Fix $n_{\ell}>0,0\le \delta \le 1$. For each user $u\in [\s{M}]$, suppose $d_{\ell}$ items are chosen for recommendation (as in Line \ref{line:sample} of Alg. \ref{algo:simple}) in phase $\ell$. With probability at least $1-\delta$, for every $r$-column in $\ca{B}^{(\ell)}$, we obtain noisy observations from at least $n_{\ell}$ distinct users provided $\frac{d_{\ell}}{\s{N}-d_{\ell}} \ge \Big(\s{M}^{-1}(12\log\delta^{-1}+r\log \s{N}+2n_{\ell})\Big)^{1/r}$.
%\begin{align}\label{eq:sufficient_1}
%    \frac{d_{\ell}}{\s{N}-d_{\ell}} \ge \Big(\s{M}^{-1}(12\log\delta^{-1}+r\log \s{N}+2n_{\ell})\Big)^{1/r}.
%\end{align}
\end{lemma}

We choose $n_{\ell}=\widetilde{O}(C(r)2^{2\ell})$ where $C(r)=2^{r+1}r^{1+r/2}$. That implies that a sufficient value of $d_{\ell}=O(\s{N}\s{M}^{-1/r}C(r)2^{2\ell})$ for which the high probability event in Lemma \ref{lem:sample_carefully} holds true.
Once, we have computed an estimate $\widetilde{\mu}_{\ca{J}}$ of $\mu_{\ca{J}}$ for all surviving sets of $r$-columns $\ca{J}\in\ca{B}^{(\ell)}$, we can eliminate $r$-columns from $\ca{B}^{(\ell)}$ to compute $\ca{B}^{(\ell+1)}$ based on a high probability confidence width $\epsilon_{\ell}$ (Line \ref{line:eliminate} in Alg. \ref{algo:simple}). More precisely, at the end of the $\ell^{\s{th}}$ phase, we have  
\begin{align}\label{eq:eliminate_new}
    \ca{B}^{(\ell+1)} = \{\ca{J}\in \ca{B}^{(\ell)}\mid \widetilde{\mu}_{\ca{J}} \ge \max_{\ca{J}'\in \ca{B}^{(\ell})}\widetilde{\mu}_{\ca{J}'}-2\epsilon_{\ell}\}
\end{align}
where $\epsilon_{\ell} = 2^{-\ell}$ with high probability, we can show that the set of hott items $\ca{A}$ always survive and furthermore, the sub-optimality gap of surviving $r$-columns decreases exponentially in every phase. Our argument can also be associated with the regret incurred by each user at every round. More precisely, we have for any user $u\in[\s{M}]$:
   \begin{align}
    &\fl{R}_{u\pi_u(1)}-\max_{j\in [r]}\fl{R}_{u\rho_u(t,j)} \\
    &\le 6r^{5/2}\Big(\frac{\s{det}^2(\fl{V}_{\ca{A}})-\s{det}^2(\fl{V}_{\ca{J}})}{\s{det}^2(\fl{V}_{\ca{A}})}\Big).
    \end{align}
Thus we can bound the regret incurred by the users at each phase as we keep eliminating sub-optimal $r$-columns. Since sub-optimal $r$-columns are eliminated in each phase, we can also show the regret to decrease exponentially with phase index. Combining all the arguments, we obtain the regret guarantee provided in Theorem \ref{thm:second}.

\section{Conclusion and Future Works}
In this work we studied online low rank matrix completion under the hott item assumption. Our framework is a significant relaxation of the settings studied in prior works that derive theoretical guarantees. We provide two novel algorithms with theoretical guarantees that exploit the hott items in distinct ways. The bounds in this paper often improve or closely compare with bounds that are derived in much restrictive setups. Though our main focus was on the statistical rate, we seek to improve the computational complexity of our methods in future (see for example Remark \ref{rmk:run}). 
%Another important future direction is further relax assumptions and extend our line of work to tensors with low rank.

\section*{Impact Statement}
The work is of theoretical nature and we do not see any negative societal consequences.

\bibliography{refs.bib}
\bibliographystyle{plain}

\section{Other Related Work}\label{app:related}

There are several other loosely related works/results that we briefly survey in this section.

% A very similar setting was considered in \cite{sen2017contextual} where the authors considered a multi-itemed bandit problem with $\s{L}$ contexts and $\s{K}$ items with context dependent reward distributions. The authors assumed that the $\s{L}\times\s{K}$ reward matrix is low rank and can be factorized into non-negative components which allowed them to use recovery guarantees from non-negative matrix factorization. Moreover, the authors only showed ETC algorithms that resulted in $\s{T}^{2/3}$ regret guarantees. Our techniques can be used to improve upon the existing guarantees in \cite{sen2017contextual} in two ways 1) Removing the assumption of the low rank components being non-negative as we use matrix completion with entry-wise error guarantees. 2) The dependence on $\s{T}$ can be improved from $\s{T}^{2/3}$ to $\s{T}^{1/2}$ when the reward matrix $\fl{R}$ is rank-$1$.

The problem of multi-dimensional bandit optimization where the goal is to assume a structure among the items (such as a matrix or tensor with low rank) was introduce in  \cite{katariya2017stochastic, katariya2017bernoulli, trinh2020solving}. In the simplest version of this  problem, in the rank-$1$ setting at each round $t\in [\s{T}]$, an agent selects one row and one column and receives a reward corresponding to the entry of a rank-1 matrix. The regret in this setting is defined with respect to the best \textit{(row,column)} pair, which corresponds to the best item. Since then, this setting has been extended to include rank $r$ \cite{kveton2017stochastic, stojanovic2024spectral} which is most relevant to our setting in terms of techniques. Furthermore, this line of work also extends to rank 1 multi-dimensional tensors \cite{hao2020low}, bilinear bandits \cite{jun2019bilinear, huang2021optimal}, and generalized linear bandits \cite{lu2021low}. However, these guarantees cannot be applied to our problem directly in most cases. Note that our goal is to minimize the regret for all users (rows of $\fl{R}$) jointly. In our problem, it is crucial to identify the entries (columns) of $\fl{R}$ with high rewards for each user (row), in contrast to the multi-dimensional online learning problem, where only the entry with the highest reward needs to be inferred.

% Finally, note that
% \cite{dadkhahi2018alternating} considers the same problem as in our setting and provide a highly efficient algorithm based on Alternating Minimization but their algorithm do not have any regret guarantees of any sort.

A related line of work is the theoretical model for User-based Collaborative Filtering (CF) studied in \cite{Bresler:2014,Bresler:2016,Heckel:2017,Mina2019,huleihel2021learning}. To the best of knowledge, these papers first motivated and theoretically analyzed the collaborative framework with the practically relevant restriction that the same item cannot be recommended more than once to the same user. 
Here, similar to \cite{https://doi.org/10.48550/arxiv.2301.07040}, a latent cluster structure is assumed to exist across users such that users in same cluster have identical expected rewards. 
These models are quite restrictive in a theoretical sense as they provide  guarantees only on a very relaxed notion of regret (termed \textit{pseudo-regret}). 

In the offline setting, several papers have studied the problem of low rank matrix completion with partially observed entries    \cite{negahban2012restricted,chen2019noisy,Montanari12,abbe2020entrywise,jain2013low,jain2017non} and also in the presence of side information such as social graphs or similarity graphs \cite{ahn2018binary,ahn2021fundamental,elmahdy2020matrix,jo2021discrete,zhang2022mc2g}. 
Some of these results can be converted into greedy algorithms for the online problem (as has been done in \cite{jain2022online}) - we can solve the offline problem in a particular number of rounds and subsequently commit to the computed estimate of the reward matrix.

\section{Experiments}

\begin{algorithm*}[t]
\caption{\textsc{PhasedElimSimplified} (A simplified version of Alg. \textsc{PhasedClusterElim})    }\label{algo:phased_elim_simplified}
\begin{algorithmic}[1]
\REQUIRE Number of users $\s{M}$, number of items $\s{N}$. Input parameter $\s{B}$. Number of hott item vectors $d$, Gap Parameter $\Delta$.

 \STATE For $\s{B}$ rounds, for each user $u\in [\s{M}]$, recommend a random item in $[\s{N}]$. Compute an estimate $\widetilde{\fl{P}}$ of $\fl{P}$. 
 \STATE For each user $u\in [\s{M}]$, compute $\ca{T}_u = \{j' \in [\s{N}]\mid \max_{j \in [\s{N}]}\widetilde{\fl{P}}_{uj}-\widetilde{\fl{P}}_{uj'} \le \Delta \}$.
 
 \STATE Run $k$-means algorithm on the set of users $[\s{M}]$ where each user $u$ is assigned the embedding $\widetilde{\fl{P}}[u][\cup_{u\in [\s{M}]}\ca{T}_u]$ i.e. the corresponding rewards estimate restricted to the items in $\cup_{u\in [\s{M}]}\ca{T}_u$.

 \STATE For each cluster of users $\ca{C}_i$ obtained, we compute $\ca{N}^{(i,0)}=\cap_{u\in \ca{C}_i}\ca{T}_u$ (or a robust version where an item needs to survive in $70\%$ of $\ca{T}_u$'s
 for $u\in \ca{C}_i$ to survive in $\ca{N}^{(i,0)}$).

\FOR{$\ell=1,2,\dots,$}

\STATE Set $\s{B}\leftarrow 3\s{B}$ and $\Delta=\Delta/3$
\STATE For $\s{B}$ rounds, for each $i\in[d]$ and for each user $u\in \ca{C}_i$, recommend a random item in $\ca{N}^{(i,\ell-1)}$. 

\STATE For each cluster of users $\ca{C}_i$ , use a standard matrix completion method to compute an estimate $\widetilde{\fl{P}}^{(i,\ell)}\in \bb{R}^{\s{M}\times \s{N}}$.

\STATE For each cluster $\ca{C}_i$ and for each user $u$ in $\ca{C}_i$, compute $\ca{T}_u = \{j' \in \ca{C}_i \mid \max_{j \in \ca{N}^{(i)}}\widetilde{\fl{P}}_{uj}-\widetilde{\fl{P}}_{uj'} \le \Delta \}$
  
%\STATE Find $d$ largest possible clusters of users $\ca{C}_1,\ca{C}_2,\dots,\ca{C}_d$ (not necessarily disjoint) such that $\ca\fl{P}_{u \in \ca{C}_i} \ca{T}_u \neq \phi$ for all $i\in [d]$ and $\cu\fl{P}_{i \in [d]} \ca{C}_i = [\s{M}]$.

\STATE For each $i\in [d]$, compute $\ca{N}^{(i,\ell)} = \ca\fl{P}_{u \in \ca{C}_i} \ca{T}_u$.

\ENDFOR

\end{algorithmic}
\end{algorithm*}

We conduct synthetic experiments to demonstrate the efficacy of our \textsc{PhasedClusterElim} algorithm that runs in phases. 
All experiments have been conducted on Colab Machines with a system RAM of 12GB and hard disk space of 23.7GB.

\paragraph{Simplified Algorithm:}
Having practical considerations, we propose a simplified version of our \textsc{PhasedClusterElim} algorithm (denoted as Algorithm \textsc{PhasedElimSimplified} - ALg. \ref{algo:phased_elim_simplified}). There are two main advantages of Algorithm \ref{algo:phased_elim_simplified} from implementation point of view - 1) the runtime is polynomial in the rank of the reward matrix 2) the number of hyperparameters to tune is much small.
At a high level, in Algorithm \ref{algo:phased_elim_simplified}, we proposed a simplification where the clustering step is decoupled (Step 3) from the exploration/exploitation step (Steps 6-11). In other words, the clustering step is done only once at the beginning. Note that such an algorithm (Algorithm 3) is simplified (easy to understand and implement) and hyperparameter-scarce. On the other hand, because of the decoupling step, it will not share the near optimal regret guarantees of  Algorithm \ref{alg:full}. Nevertheless, Algorithm \ref{algo:phased_elim_simplified} is expected to be significantly better than the standard greedy algorithm because of the initial clustering step. 

 Let us elaborate. In Algorithm \textsc{PhasedElimSimplified}, the overall idea is to first recommend random items to every user in the first few rounds and then compute an initial estimate of the low rank matrix. Based on this initial estimate, we run $k$-means with $k=d$ and cluster the set of users into $d$ clusters. For each cluster of users, we find a small active set of items that are good for all users in the same cluster. The main workhorse is that users in the same cluster have the same hott item as the best item.
Next, Algorithm \textsc{PhasedElimSimplified} runs in phases of exponentially increasing length where for each cluster of users, we again recommend random items (from the active set) to every user in the same cluster. Next, we run low rank matrix completion to compute an estimate of the sub-matrix restricted to each cluster of users and its corresponding good items. Based on this estimate, we prune the set of good items further. The goal is to converge to the best item for every cluster of users.

\noindent \textbf{Baselines for Comparison:} We compare with two distinct baselines 1) Alternating Minimization (AM) - This algorithm was proposed for online low rank matrix completion  in \cite{dadkhahi2018alternating} based on an alternating minimization recipe. Here, assuming the reward matrix to be of low rank i.e.  $\fl{P}=\fl{U}\fl{V}^{\s{T}}$, estimates of $\fl{U},\fl{V}$ are improved gradually.
However, it is important to note that no theoretical guarantees were provided in \cite{dadkhahi2018alternating}. 2) Explore Then Commit (ETC) Algorithm - This greedy algorithm was proposed in \cite{jain2022online}.
In this algorithm, we recommend random items to every user for an initial few rounds (say $m$ - exploration rounds) and compute an estimate of the expected reward matrix. Based on the low rank estimate, we recommend the estimated best item for every user in the remaining rounds.

We consider an instance with $\s{M}=200$ users, $\s{N}=200$, items, rank $r=4$ across $\s{T}=300$ rounds. Here, we consider the expected reward matrix $\fl{P}=\fl{U}\fl{V}^{\s{T}}$ constructed in the following way: 
1) For every $i\in [\s{M}]$, in the $i^{\s{th}}$ row of $\fl{U}$, there is a \texttt{1} in the $(i\%\s{C})^{\s{th}}$ position and \texttt{0} elsewhere in the row. 2) The first $\s{N}-r$ columns of $\fl{V}$ are obtained by generating each entry from $\ca{N}(0,1)$ independently and subsequently, the negative entries are trimmed to zero. Suppose the maximum positive entry in the first $\s{N}-r$ columns of $\fl{V}$ is $\alpha$. In that case, the final $r$ columns are an identity matrix multiplied by a factor of $2\alpha$. This ensures the hott item property - every column in $\fl{V}$ can be written as a convex combination of the final $r$ columns. Now, with this reward matrix, at each round, we recommend one item to every user such that a noisy reward is obtained (eq. \eqref{eq:obs}) with gaussian noise having variance $\sigma^2=0.25$. 

We run Algorithm \textsc{PhasedElimSimplified} (Alg. \ref{algo:phased_elim_simplified} - the simplified version of Alg. \textsc{PhasedClusterELim}) in this set-up (with an initial exploration period of $25$) along with the AM algorithm and the ETC algorithm.
The number of exploration rounds in ETC is a hyper-parameter and we run the ETC algorithm with number of exploration rounds $m=25,50,75$.
For the AM algorithm, we use the same hyper-parameters as suggested in \cite{dadkhahi2018alternating} for the setting when $\fl{U},\fl{V}$ have entries sampled independently from $\ca{N}(0,1)$. 

\textbf{Results and Insights:} All algorithms have been run $5$ times and the average regret has been reported across the $5$ runs in Figures \ref{fig:average_regret_gap1} and \ref{fig:time_comparison}. Note that from Figure \ref{fig:average_regret_gap1} our algorithm incurs a much lower cumulative regret compared to the Alternating Minimization algorithm and the ETC (greedy) algorithm with different exploration periods (small and large). The picture is clearer in \ref{fig:time_comparison} - both the ETC algorithms and our algorithm have an initial exploratory period but our algorithm is able to make use of the exploration period significantly better. Note that the ETC algorithm with exploration period=25 rounds has a higher regret per round than our algorithm with a similar exploration period=25 rounds. ETC algorithms with higher exploration period have smaller exploitation regret per round but their exploration cost is larger. The AM algorithm also has a similarly large regret.

\begin{figure*}[!t]
  \begin{subfigure}[t]{0.47\textwidth}
    \centering 
    \includegraphics[scale = 0.34]{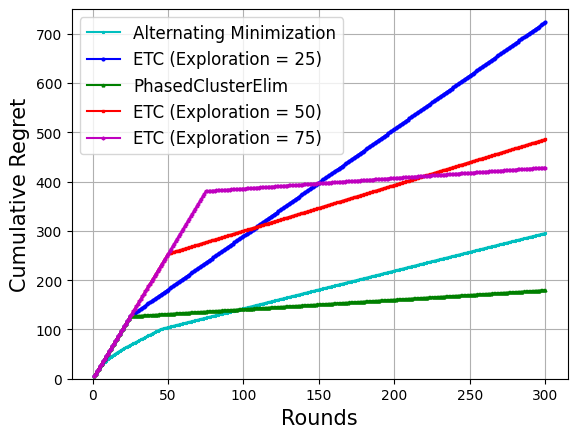}\vspace*{-5pt}
    \caption{Regret accrued by Alg. \ref{algo:phased_elim_simplified}, Explore-Then-Commit algorithm \cite{jain2022online} for exploration periods ($25,50,75$), Alternating Minimization (AM) algorithm \cite{dadkhahi2018alternating} when $\s{T}=300$. The cumulative regret is plotted with number of rounds.}
 ~\label{fig:average_regret_gap1}
  \end{subfigure}
  \hfill
\begin{subfigure}[t]{0.47 \textwidth}
\centering
   \includegraphics[scale = 0.34]{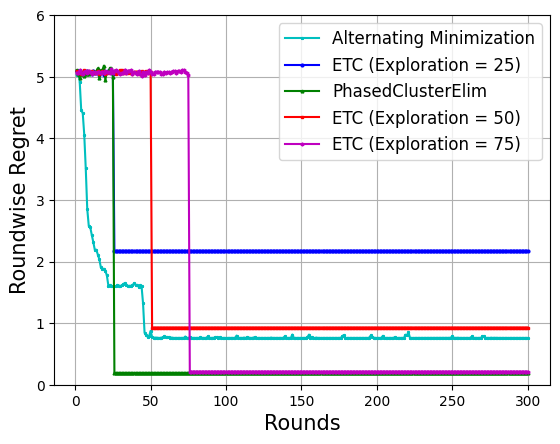}\vspace*{-5pt}
   \caption{\small Comparison of the regret incurred in each round ($\s{M}^{-1}\sum_{u \in [\s{M}]}\fl{P}_{u\pi_u(1)}- \frac{1}{\s{M}}\sum_{u\in[\s{M}]}\fl{P}_{u\rho_u(t)}^{(t)}$) by Algorithm \ref{algo:phased_elim_simplified} and compared with Explore-Then-Commit algorithm \cite{jain2022online} for exploration periods ($25,50,75$), Alternating Minimization (AM) algorithm \cite{dadkhahi2018alternating}.}
       ~\label{fig:time_comparison}
 \end{subfigure}%
%\vspace*{-15pt}
 \caption{Comparison of regret incurred by Algorithm~\ref{algo:phased_elim_simplified} (a simplified version of our proposed phased elimination algorithm namely Algorithm \textsc{PhasedClusterElim}) with baselines such as Explore-Then-Commit (Greedy) algorithm and Alternating Minimization (AM) algorithm. Clearly, Algorithm \ref{algo:phased_elim_simplified} outperforms the other baselines significantly.}

\end{figure*}

%Some of these results namely the ones that provide $\|\cdot\|_{\infty}$ norm guarantees on the estimated matrix can be adapted into  Explore-Then-Commit (ETC) style algorithms (see Sec. \ref{sec:witemup}). %However, these guarantees will be sub-optimal in the number of rounds $\s{T}$ (see Section \ref{sec:witemup}). 
% Finally, there is significant amount of related theoretical work for online non-convex learning \citep{suggala2020online, yang2018optimal, huang2020online} and  empirical work for online Collaborative Filtering \citep{huang2020online,lu2013second,zhang2015simple} but they do not study the regret in online matrix completion setting.  

\begin{algorithm}[!t]
\caption{\texttt{CheckIndependentSet}: Algorithm to check existence of an independent set of size $d$}
\label{alg:is}
\textbf{Input}: A graph G
\begin{algorithmic}[1] 
\FOR{any combination of $d$ vertices $(u_1,\ldots,u_d)$ from the graph}
    \IF{there is no edge between any pair $u_i$ and $u_j$}
        \STATE \textbf{return} (True , $(u_1,\ldots,u_d)$)
    \ENDIF
\ENDFOR
\STATE  \textbf{return} False

\end{algorithmic}
\end{algorithm}

\begin{algorithm}[!t]
\caption{\texttt{ExpanditemSets}: Algorithm to update the items sets of users that are connected to multiple clusters}
\label{alg:clique}
\textbf{Input}: A graph $G(V,E)$; Candidate item sets $\cT^\ell_u$ for every user; Seed users $u_1,\ldots,u_d$, suboptimality level $c_{\ell-1}$, estimated reward matrix $\tilde R^{\ell}$
\begin{algorithmic}[1] 
\FOR{each user $u$ that is connected to more than one user in $u_{1:d}$}
    \STATE Let $\ca{C} \leftarrow \{ u_i : \cT^\ell_{u_i} \cap \cT_u^\ell \neq \Phi \}$

    \textcolor{blue}{\texttt{//Expanding the item set for users in $D$ slightly}}
    
    \STATE Set $\cT_u^\ell \leftarrow \{x : \tilde R_{u,\tilde \pi_u(1)}^\ell - \tilde R_{u,x}^\ell \le 12k_{\ell-1} \text{ and } x \in \cT^\ell_{C_i} \text{ for some } i \in \ca{C} \}$
\ENDFOR

\STATE \textbf{return} {the updated item sets}

\end{algorithmic}
\end{algorithm}

\begin{algorithm}[!t]
\caption{\textsc{Estimate} from \cite{jain2022online}   \label{algo:estimate}}
\begin{algorithmic}[1]
\REQUIRE Set of users $\ca{U}\subseteq[\s{M}]$, set of items $\ca{V}\subseteq[\s{N}]$, sub-optimality level $\eta$, sampling probability $p$,  regularization parameter $\lambda$, rounds control parameter $s$, candidate item sets of users $\cT_{u}$ for all $u \in [M]$. Index of round $t$ is relative to the first round when the algorithm is invoked; hence $t=1,2,\dots,m$.

%\STATE Initialize $\Omega = \phi$.
\STATE For each tuple of indices $(i,j)\in \ca{U}\times \ca{V}$, independently set $\delta_{ij}=1$ with probability $p$ and $\delta_{ij}=0$ with probability $1-p$.

\STATE  Denote $\Omega=\{(i,j)\in \ca{U}\times \ca{V} \mid \delta_{ij}=1\}$ and
 $b=\max_{i \in [ \ca{U}]}\mid |j \in [ \ca{V}]\mid (i,j) \in \Omega|$ to be the maximum number of index tuples in a particular row.
 \STATE Set $b = |\ca{V}|$ and $m = bs$.
\FOR{$\ell=1,2,\dots,m/b$} 
\STATE For all $(i,j)\in \Omega$, set $\s{Mask}_{ij}=0$.
\FOR{$\ell'=1,2,\dots,b$}
\FOR{each user $u\in \ca{U}$ in round $t=(\ell-1)b+\ell'$}
\STATE Recommend an item $\rho_u(t)$ in $\{j \in \ca{V}\mid (u,j)\in \Omega, \s{Mask}_{uj}=0\}$ and set $\s{Mask}_{u\rho_u(t)}=1$. 
If not possible then recommend any item $\rho_u(t)$ in $\ca{V}$ s.t. $(u,\rho_u(t))\not\in\Omega$.
Observe $\fl{R}^{(t)}_{u\rho_u(t)}$.
\ENDFOR
\FOR{each user $u\notin \ca{U}$ in round $t=(\ell-1)b+\ell'$}
\STATE Recommend an item randomly from its candidate item set $\cT_u$ and observe its noisy reward.
\ENDFOR
\ENDFOR
\ENDFOR

\STATE For each $(u,j)\in \Omega$, compute $\fl{Z}_{uj}$ to be average of $\lfloor m/b \rfloor$ observations corresponding to user $u$ being recommended item $j$ i.e. $\fl{Z}_{uj} = \text{avg}\{\fl{R}^{(t)}_{u\rho_u(t)} \text{ for }t\in [m]\mid \rho_u(t)=j\}$. 
Discard all other observations corresponding to indices not in $\Omega$.
%\STATE Simulate observations of all marked entries in $\fl{Z}$ that are not in $\ca{U}\times \ca{V}$ by setting them to be zero.

\STATE Without loss of generality, assume $|\ca{U}| \le |\ca{V}|$. For each $i\in \ca{V}$, independently set $\delta_i$ to be a value in the set $[\lceil|\ca{V}|/|\ca{U}|\rceil]$ uniformly at random. Partition indices in $\ca{V}$ into $\ca{V}^{(1)},\ca{V}^{(2)},\dots,\ca{V}^{(k)}$ where $k=\lceil|\ca{V}|/|\ca{U}|\rceil$ and $\ca{V}^{(q)}=\{i\in \ca{V}\mid \delta_i =q\}$ for each $q\in [k]$. Set $\Omega^{(q)}\leftarrow \Omega \cap (\ca{U}\times \ca{V}^{(q)})$ for all $q\in [k]$. \#\textit{If $|\ca{U}| \ge |\ca{V}|$, we partition the indices in $\ca{U}$}. 

\FOR{$q\in [k]$}
\STATE Solve convex program 
\vspace*{-15pt}
\begin{align}\label{eq:convex}
    \min_{\fl{Q}^{(q)}\in \bb{R}^{|\ca{U}|\times|\ca{V}^{(q)}|}} \frac{1}{2}\sum_{(i,j)\in \Omega^{(q)}}\Big(\fl{Q}^{(q)}_{i\pi(j)}-\fl{Z}_{ij}\Big)^2+\lambda\|\fl{Q}^{(q)}\|_{\star},
\vspace*{-15pt}    
\end{align}
where $\|\fl{Q}^{(q)}\|_{\star}$ denotes nuclear norm of matrix $\fl{Q}^{(q)}$ and $\pi(j)$ is index of $j$ in set $\ca{V}^{(q)}$. 
\ENDFOR
\STATE \textbf{return}  $\widetilde{\fl{Q}}\in \bb{R}^{|\ca{U}|\times |\ca{V}|}$ s.t. $\widetilde{\fl{Q}}_{\ca{U},\ca{V}^{(q)}}=\fl{Q}^{(q)}$ for all $q\in [k]$ and for every  $(i,j)\not \in \ca{U}\times \ca{V}$, $\widetilde{\fl{Q}}_{ij}=0$.
\end{algorithmic}
\end{algorithm}

\section{Missing Algorithmic Modules and Proofs} \label{app:proofs}
The missing algorithmic modules from Section \ref{sec:difficult} are presented in Algorithms \ref{alg:is}, \ref{alg:clique} and \ref{algo:estimate}.

\subsection{Offline Low Rank Matrix Completion Oracle (MCO)}
We provide an MCO in Algorithm \ref{algo:estimate} which is essentially the same as in \cite{jain2022online}. We have the following matrix completion guarantee for Algorithm \ref{algo:estimate} which is a direct consequence of Remark 6 from \cite{jain2022online}. 

\begin{prop}[Matrix completion guarantee] \label{prop:mat1}
Let rank $r=O(1)$ matrix $\fl{R}\in \mathbb{R}^{\s{M}\times \s{N}}$ with SVD $\fl{R}=\bar{\fl{U}}\f{\Sigma}\bar{\fl{V}}^{\s{T}}$ satisfy $\left|\left|\bar{\fl{U}}\right|\right|_{2,\infty}\le \sqrt{\mu r /\ s{M}}$, $\left|\left|\bar{\fl{V}}\right|\right|_{2,\infty}\le \sqrt{\mu r /\ s{N}}$ and condition number $\kappa=O(1)$.
Let $d_1 = \max(\s{M},\s{N})$ and $d_2 =\min(\s{M},\s{N})$.  Set $p=C\mu^2 d_2^{-1}\log^3 d_2$, $s=\Big\lceil \Big(\frac{c\sigma r \sqrt{\mu}}{k_\ell \log d_2}\Big)^2 \Big\rceil$ and $\lambda = C_{\lambda} \sigma \sqrt{d_2p}$ for a suitable constants $c,C,C_\lambda > 0$. Let $\hat {\fl{R}} = \texttt{ESTIMATE}(\s{M},\s{N},k_\ell,p,\lambda,s)$. Then with probability at-least $1-\delta$, we have that
\begin{align}
    \| \hat {\fl{R}} - {\fl{R}} \|_\infty \le k_\ell
\end{align}

Further the total number of rounds Algorithm \ref{algo:estimate} runs is bounded by $\tilde O \left( \max(1,\s{N}/\s{M})/ k^2_{\ell} \right)$, where $\tilde O$ hides multiplicative factors of $\log(1/\delta)$. 
\end{prop}

Note that in each round in Algorithm \ref{alg:estimate}, one noisy observation is obtained from each row of the low rank matrix $\bar{\fl{R}}$.

\begin{rmk}
    Note that the matrix completion guarantee in Proposition \ref{prop:mat1} implies that a reasonable estimate of the reward matrix $\fl{R}$ can be obtained if we have $\widetilde{O}(\s{NM}sp)$ randomly sampled observations. Intuitively, every entry of $\fl{R}$ is used with probability $p$ - if used, then $s$ noisy observations are made corresponding to that entry and the average of those observations is taken to reduce the variance. The bottom-line is that we do not need to see all the entries of the reward matrix - usually $p$ is set to be of the order of $1/\min(\s{M},\s{N})$ and therefore only a small number of entries are actually sufficient for us to compute an estimate of the reward matrix.
\end{rmk}

In our algorithm, it will be often necessary to estimate sub-matrices of the reward matrix. To do so, the sub-matrix like the entire reward matrix $\fl{R}$ needs to satisfy certain incoherence and condition number guarantees - they need to be small. In this regard, we need to ensure that conditions in Proposition \ref{prop:mat1} are satisfied for the different sub-matrices of the low rank reward matrix that we have at hand. Assumption \ref{assum:matrix2} (described below) ensures that the incoherence and condition number of all sub-matrices of a reasonable size are bounded by at most a constant multiplicative factor of the incoherence and condition number of the original reward matrix $\fl{R}$ (see \cite{https://doi.org/10.48550/arxiv.2301.07040} for example).  

\begin{assumption}[\textbf{A3} Assumptions on reward matrix $\fl{R}$]\label{assum:matrix3}
We assume that $\fl{R}$ with SVD decomposition $\fl{R}=\widetilde{\fl{U}}\f{\Sigma}\widetilde{\fl{V}}^{\s{T}}$ satisfies the following  properties 1) (Condition Number) $\fl{R}$ has rank $r$ and has non zero singular values $\lambda_1>\lambda_2 > \dots > \lambda_{r}$ with $\lambda_1/\lambda_{r}=O(1)$ 2) ($\mu$-incoherence)  $\lr{\widetilde{\fl{U}}}_{2,\infty}\le \sqrt{\mu r/\s{N}}$ and $\lr{\widetilde{\fl{V}}}_{2,\infty}\le \sqrt{\mu r/\s{M}}$ for some $\mu=O(1)$. 3) (Subset Strong Smoothness (a)) For some constant $\beta>0$ and for any subset of indices $\ca{S}\subseteq [\s{N}], \ca{S}= \ca{T}^{(j)}$ (corresponding to some cluster of users), we must have $\fl{x}^{\s{T}}\widetilde{\fl{U}}_{\ca{S}}^{\s{T}}\widetilde{\fl{U}}_{\ca{S}}\fl{x} \ge \beta $ for all unit norm vectors $\fl{x}\in \bb{R}^{r}$.
4) (Subset Strong Smoothness (b)) For some $\alpha$ satisfying $\alpha \log \s{M}=\Omega(1)$, $\gamma =\widetilde{O}(1)$, for any subset of indices $\ca{S}\subseteq [\s{M}], |\ca{S}| \ge \gamma r$, we must have $\fl{x}^{\s{T}}\widetilde{\fl{V}}_{ \ca{S}}^{\s{T}}\widetilde{\fl{V}}_{ \ca{S}}\fl{x} \ge \alpha \left|\ca{S}\right|/\s{M}$ for all $\fl{x}\in \bb{R}^{r},\lr{\fl{x}}_2=1$. 
\end{assumption}

%Assumption \ref{assum:matrix2}, also proposed in \cite{https://doi.org/10.48550/arxiv.2301.07040} states that matrix $\fl{R}$ must be well-conditioned and incoherent - two properties that are commonly used for invoking theoretical guarantees on offline low rank matrix completion \cite{chen2019noisy,jain2013low}. In essence, Assumption \ref{assum:matrix2} allows us to invoke low rank matrix completion guarantees to suitably large sub-matrices of $\fl{R}$.
The subset strong smoothness assumptions on $\fl{U},\fl{V}$ suffice to prove that any sub-matrix of $\fl{R}$ of a reasonable size also satisfies the low condition number and incoherence properties \cite{https://doi.org/10.48550/arxiv.2301.07040}. Furthermore, similar smoothness assumptions (Statistical RIP property) were made in \cite{sen2017contextual} where a greedy algorithm was provided for optimizing the regret $\s{Reg}$ based on non-negative matrix factorization. Such an implication ensures that we can invoke Proposition \ref{prop:mat1} for the sub-matrices of the low rank reward matrix $\fl{R}$ with the same incoherence parameter $\mu$ and condition number $\kappa=O(1)$.
Thus, we can assume that the reward matrix satisfies Assumptions \ref{assum:hott2}, \ref{assum:simplicity} and \ref{assum:matrix3}. A detailed discussion on the feasibility of such assumptions jointly is deferred to Section \ref{app:assum}.

%Next, we verify Assumption \ref{assum:matrix2} about the existence of MCO with desirable properties.

\subsection{Proofs for generalized regret setting (Setting 1 from Section \ref{sec:problem})}

Throughout the proofs, we suppose that the reward matrix satisfies the assumptions detailed in Section \ref{sec:problem} for Setting 1. However, we generalise Assumption \ref{assum:hott} as below so that we allow $d$ number of hott-topics where $d \ge r$. All the results of Setting 1 can be recovered by setting $d=r$, where recall that $r$ is the rank of the reward matrix.

\begin{assumption}[\textbf{B1} Hott items and vectors]\label{assum:hott2}

We assume that there exists a set of $d$ unknown distinct ordered indices $\ca{A}\subseteq [\s{\s{N}}], \left|\ca{A}\right|=d$ such that all the item embedding vectors $\{\fl{V}_i\}_{i\in [\s{\s{N}}]}$ lie within the convex hull of $\{\fl{V}_i\}_{i\in \ca{A}}\cup\{\fl{0}\}$.
 We will call the items corresponding to indices in $\ca{A}$ hott items and the rows in $\fl{V}_{\ca{A}}$ to be hott vectors. Further, we suppose that $r = O(1)$.
\end{assumption}

First we proceed to bound the regret of stage 1 in Algorithm \ref{alg:full}

\begin{lemma} \label{lem:stage1-best-item}
Let the phase $\ell$ be such that the Algorithm \ref{alg:full} is operating in Stage 1. Then with probability at-least $1-\delta$, we have that $\pi_u(1) \in \cT_u^{\ell+1}$ for all users $u$.
\end{lemma}
\begin{proof}
Consider a user $u$. We have
\begin{align}
    \tilde {\fl{R}}^{\ell+1}_{u,\tilde \pi_u(1)} - \tilde {\fl{R}}^{\ell+1}_{u,\pi_u(1)}
    &= \tilde {\fl{R}}^{\ell+1}_{u,\tilde \pi_u(1)} - {\fl{R}}_{u,\tilde \pi_u(1)} + {\fl{R}}_{u,\pi_u(1)} - \tilde {\fl{R}}^{\ell+1}_{u,\pi_u(1)} + {\fl{R}}_{u,\tilde \pi_u(1)} - {\fl{R}}_{u,\pi_u(1)}\\
    &\le 2k_\ell,
\end{align}
with probability at-least $1-\delta$ where in the last line we used Proposition \ref{prop:mat1} and the fact that $ {\fl{R}}_{u,\tilde \pi_u(1)} - {\fl{R}}_{u,\pi_u(1)} \le 0$.

\end{proof}
Lemma \ref{lem:stage1-best-item} implies that with high probability, in any epoch that belongs to Stage 1, the best item for each user survives in its candidate set.

The following terminology will be central to the arguments in the proofs.

\textbf{Terminology:} We use the terms phase and epoch interchangeably. For the rest of the proof, we fix the quantities. Let $\ell_0$ be the value of epoch $\ell$ in Algorithm \ref{alg:full} when the control reaches line 7. For epochs $\ell < \ell_0$, we fix $k_\ell: = 2^{-\ell}$ for $\ell \le \ell_0$; $k_{\ell_0} =  k_{\ell_0-1}/10$; $k_\ell = k_{\ell-1}/2$ for $\ell \ge \ell_0+1$. Let $c_\ell = 3 k_\ell$ used in Algorithm \ref{alg:full}. Stage 1 runs from epochs $0$ to $\ell_0$; All epochs greater than or equal to $\ell_0+1$ belongs to Stage 2. For notational simplicity we take the hott-topic set $\ca{A} = [d]$.

The following Proposition is a direction consequence of Definition \ref{defn:weird_gap}.

\begin{prop}\label{ass:opinion}
There exists $d$ opinionated users  $v_{1:d}$ such that for any $v_i$, any hott-topic $x \neq \pi_{v_i}(1)$ is well separated:

\begin{align}
    {\fl{R}}_{v_i,\pi_{v_i}(1) } - {\fl{R}}_{v_i,x} \ge \Delta_{\text{hott}}.
\end{align}

We call the users $v_{1:d}$ as opinionated users. We remark that both $v_{1:d}$ and $\Delta$ need not be known ahead of time.
\end{prop}

\begin{lemma} \label{lem:seperation}
Consider an item $x \in [\s{M}]$, a user $u$ and epoch $\ell \le \ell_0$. Suppose $x \notin \cT_u^{\ell+1}$. Then with probability at-least $1-\delta$, we have that
\begin{align}
{\fl{R}}_{u,\pi_u(1)} - {\fl{R}}_{u,x}
&>  k_{\ell}.
\end{align}
\end{lemma}
\begin{proof}
Since $x \notin \cT_u^{\ell+1}$,
\begin{align}
3 k_\ell
&< \tilde {\fl{R}}_{u,\tilde \pi_u(1)} - \tilde {\fl{R}}_{u,x}\\
&\le_{(a)} \tilde {\fl{R}}_{u,\tilde \pi_u(1)} -  {\fl{R}}_{u,\tilde \pi_u(1)} + {\fl{R}}_{u,x} - \tilde {\fl{R}}_{u,x} + {\fl{R}}_{u,\pi_u(1)} - {\fl{R}}_{u,x}\\
&\le_{(b)} 2 k_\ell +  {\fl{R}}_{u,\pi_u(1)} - {\fl{R}}_{u,x},
\end{align} 
where in line (a) we used the fact that ${\fl{R}}_{u,\pi_u(1)} \ge {\fl{R}}_{u,\tilde \pi_u(1)}$ and in line (b) we used Proposition \ref{prop:mat1}.

Thus if an item $x$ gets eliminated by the end of epoch $\ell$, we are guaranteed with high probability that

\begin{align}
{\fl{R}}_{u,\pi_u(1)} - {\fl{R}}_{u,x}
&>  k_\ell.
\end{align}
\end{proof}

\begin{lemma} \label{lem:subopt}
Consider an epoch $\ell$ such that $x \in \cT_u^{\ell+1}$ for a user $u$. Then with probability at-least $1-\delta$, it holds that
\begin{align}
    {\fl{R}}_{u,\pi_u(1)} - {\fl{R}}_{u,x}
    &\le 5 k_\ell.
\end{align}
\end{lemma}
\begin{proof}
We have
\begin{align}
     {\fl{R}}_{u,\pi_u(1)} - {\fl{R}}_{u,x}
     &=  {\fl{R}}_{u,\pi_u(1)} - \tilde {\fl{R}}^{\ell+1}_{u, \pi_u(1)} + \tilde {\fl{R}}^{\ell+1}_{u, x} - {\fl{R}}_{u,x} + \tilde {\fl{R}}^{\ell+1}_{u, \pi_u(1)} - \tilde {\fl{R}}^{\ell+1}_{u, \tilde \pi_u(1)} + {\fl{R}}^{\ell+1}_{u, \tilde \pi_u(1)} -  \tilde {\fl{R}}^{\ell+1}_{u, x}\\
     &\le 2k_\ell + 3k_\ell\\
     &=5k_\ell,
\end{align}
with probability at-least $1-\delta$, where we used Proposition \ref{prop:mat1} and the fact that $x \in \cT_u^{\ell+1}$

\end{proof}

\begin{lemma} \label{lem:stage1-epochs}
Consider an epoch $\ell_0$ such that the call to Algorithm \ref{alg:is} in Line 2 of Algorithm \ref{alg:full} returns True for the first time. Let $u_{1:d}$ be the users returned by Algorithm \ref{alg:is}. Then it holds with probability at-least $1-\delta$ that:
\begin{enumerate}
    \item Best item for user $u_i$ is different from the best item for user $u_j $, $j \neq i$.

    \item $\ell_0   \le \lceil{ \log \left( \frac{40}{\Delta_{\text{hott}}} \right)}\rceil$, where $\Delta_{\text{hott}}$ is as in Proposition \ref{ass:opinion}.
    
\end{enumerate}
\end{lemma}
\begin{proof}
By Lemma \ref{lem:stage1-best-item}, the best item for any user $u$ survives in its candidate set $T_u^{\ell_0}$ with probability at-least $1-\delta$. Then if two users $u_i$ and $u_j$ have the same best item, then there will be an edge between them. This is a contradiction to the fact that $u_{1:d}$ forms an independent set. This proves the first statement.

Consider opinionated users $V := \{v_1,\ldots,v_d \}$ as per Proposition \ref{ass:opinion}. Epoch $\ell_0$ is the first time Line 2 of Algorithm \ref{alg:full} returns True. So at epoch $\ell_0-1$ there must exist a user $u \in V$ and a hott-topic $x \neq \pi_u(1)$ such that $x \in \cT_u^{\ell_0-1}$. Now combining Proposition \ref{ass:opinion} and Lemma \ref{lem:subopt} we have

\begin{align}
    \Delta_{\text{hott}}
    &\le {\fl{R}}_{u,\pi_u(1)} - {\fl{R}}_{u,x}\\
    &\le 5k_{\ell_0-2},
\end{align}
with probability at-least $1-\delta$. Rearranging the last display yields the second statement of the lemma.

\end{proof}

Next, we provide a crude upper bound on Stage 1 regret of Algorithm \ref{alg:full}

\begin{lemma}[Stage 1 regret bound] \label{lem:stage1-regret}
Let $\ell_0$ be the number of epochs until the completion of Stage 2. Then with probability at-least $1-\delta$, the average regret incurred by Algorithm \ref{alg:full} during Stage 1 is bounded by
\begin{align}
    \text{Reg}_{\text{Stage 1}}
    &= \tilde O(\max(1,\s{N}/\s{M}) / \Delta_{\text{hott}}^2 + 1),
\end{align}
with probability at-least $1-\delta$. Here $\Delta_{\text{hott}}$ is as in Proposition \ref{ass:opinion}.
\end{lemma}
\begin{proof}
Due to Proposition \ref{prop:mat1}, the total number of rounds Stage 1 operates can be bounded by

\begin{align}
\sum_{l=1}^{\ell_0}  \tilde O \left( \max(1,\s{N}/\s{M})/ k^2_{\ell} \right)
&= \sum_{l=1}^{\ell_0}  \tilde O \left( \max(1,\s{N}/\s{M}) 2^{2\ell} \right)\\
&= \tilde O \left( \max(1,\s{N}/\s{M}) 2^{2\ell_0} +1 \right)\\
&= \tilde O(\max(1,\s{N}/\s{M}) / \Delta_{\text{hott}}^2+1),
\end{align}
where the last line uses Lemma \ref{lem:stage1-epochs}.

In Stage 1, the algorithm incurs $O(1)$ regret for any user since the rewards are assumed to be bounded. This concludes the lemma.

\end{proof}

\begin{comment}
    
\begin{lemma} \label{lem:cluster_opinion}
Let $v_1,\ldots,v_d$ be the set of opinionated users as in Proposition \ref{assum:opinion}. Let $i \in [d]$ be the hott-topic corresponding to user $v_i$. Then it holds that $j \notin \cT_{v_i}^{\ell_0+1}$ for all $j \in [d]$ different from $i$.
\end{lemma}
\begin{proof}
Consider a user $v_i$. By Lemma \ref{lem:stage1-best-item} we have that $i \in \cT_{v_i}^{\ell_0+1}$ with probability at-least $1-\delta$.
For the sake of contradiction let's assume that $j \in T_{v_i}^{\ell_0+1}$ for some $j \in [d]$ different from $i$.

Then by Lemma \ref{lem:subopt} with probability at-least $1-2\delta$, we have that
\begin{align}
    {\fl{R}}_{u_i,i} - {\fl{R}}_{u_i,j}
    &\le 5 k_{\ell_0}\\
    &\le \Delta/2,
\end{align}
where in the last line we used Eq.\eqref{eq:delta}. Now this is a contradiction to Assumption \ref{assum:opinion}.
\end{proof}
\end{comment}

\begin{lemma} \label{lem:clique-best-item}
Suppose before the call to Algorithm \ref{alg:clique}, the best item of every user belongs to its corresponding candidate sets. Then the best item for each user survives in the candidate item sets created by the end of Algorithm \ref{alg:clique}.

Consider the seed users $u_{i:d}$ given as input to the algorithm. Since they form an independent set, the best item of these users must be different. WLOG we assume that user $u_i$ has hott-topic $i$ as its best item. 

Further consider a user $u$ that is connected to more than one user in $u_{1:d}$ at Line 1 of Algorithm \ref{alg:clique}. Let $u_i$ be one such user that the user $u$ is connected to. Then in the final item set $\cT_u^\ell$ maintained by the algorithm (Line 3), it is guaranteed that the hott-topic item $i \in \cT_u^\ell$.

\end{lemma}
\begin{proof}

The first time Algorithm \ref{alg:clique} gets called within Algorithm \ref{alg:full} is the epoch $\ell_0+1$. So the current epoch must obey $\ell \ge \ell_0+1$.

Let $u_1,\ldots,u_d$ be the set of users maintained at Line 7 of Algorithm \ref{alg:full}. These are basically given as the ``seed users'' to Algorithm \ref{alg:clique}. Let the current epoch be $\ell$. Since the best item was not eliminated before the call to Algorithm \ref{alg:clique}, we have that $i \in \cT_{u_i}^\ell$.

Note that we are in in stage 2. The independent set of users were first discovered before the start of epoch $\ell_0$. (i.e based on the items sets $\cT_{\cdot}^{\ell_0}$ ). Then due to Lemma \ref{lem:seperation}, we conclude that for any hott-topic $j \in [d] \setminus \{ i\}$
\begin{align}
    {\fl{R}}_{u_i,i} - {\fl{R}}_{u_i,j}
    &> k_{\ell_0-1} \label{eq:prev}
\end{align}

Next, we are going to show that for the user $u$,  the expansion step at Line 3 ensures that the hott-topic $i$ belongs to the updated item set of the user $u$.

Let the itemset for the user $u$ before invoking Algorithm \ref{alg:clique} be denoted as $\tilde \cT_u^\ell$. By the premise of the lemma, $\tilde  \cT_u^\ell \cap \cT_{u_i}^\ell \neq \Phi$. 

In the first case where $i \in \tilde  \cT_u^\ell$, we can surely say that $i \in  \cT_u^\ell$ as well since in Line 3 we are only setting the acceptable sub-optimality level to be 12 times the sub-optimality level used to create $\tilde  \cT_u^\ell$ in the first place.

So in the rest of the proof we consider the case when $i \not\in \tilde  \cT_u^\ell$. Since $\tilde  \cT_u^\ell \cap \cT_{u_i}^\ell \neq \Phi$, there must exist an item $x$ such that $x \in \tilde  \cT_u^\ell \cap \cT_{u_i}^\ell$.  Note that this item $x$ cannot be a hott-topic because $\cT_{u_i}^\ell$ does not contain any other hott-topic $j \neq i$.

Let $j \neq i$ be the best item for user $u$. Then by Lemma \ref{lem:subopt}, we conclude that

\begin{align}
    {\fl{R}}_{u_i,i} - {\fl{R}}_{u_i,x} \le 5k_{\ell-1}, \label{eq:x1}
\end{align}
and
\begin{align}
    {\fl{R}}_{u,j} - {\fl{R}}_{u,x} \le 5 k_{\ell-1}. \label{eq:x2}
\end{align}

Note that due to our low rank reward matrix assumption, the reward of any user $v$ for an item $k$, denoted by ${\fl{R}}_{v,k}$ can be represented as ${\fl{R}}_{v,k} = e(v)^T e'(k)$  where $e,e'$ denotes the latent embedding of user and item respectively.

Hence the previous equations can be restated as 

\begin{align}
   e(u_i)^T e'(i) - e(u_i)^T e'(x) \le 5k_{\ell-1}, \label{eq:11}
\end{align}
and
\begin{align}
    e(u)^Te'(j) - e(u)^Te'(x) \le 5 k_{\ell-1}. \label{eq:22}
\end{align}

Adding the last two displays yields that

\begin{align}
    e(u_i + u)^T e'(x-j)
    &\ge e(u_i)^Te'(i-j) - 10k_{\ell-1}\\
    &>_{(a)} k_{\ell_0-1} - 10 k_{\ell-1}\\
    &\ge k_{\ell_0-1} - 10 k_{\ell_0}\\
    &= 0,
\end{align}
where in line (a) we used Eq.\eqref{eq:prev} and in the following lines we used the fact that the current epoch $\ell \ge \ell_0+1$ since we are in stage 2 and hence $k_\ell \le k_{\ell_0} = k_{\ell_0-1}/10$ (Line 8 of Algorithm \ref{alg:full}).

Thus we conclude that for a hypothetical user defined by the embedding $e(u_i+u)$, the item $x$ is strictly better than the hott-topic $j$.

For any hott-topic $p \notin \{i,j \}$, since ${\fl{R}}_{u,p} \le {\fl{R}}_{u,j}$ we have that $e(u)^Te'(p) - e(u)^Te'(x) \le 5 k_\ell$ due to Eq.\eqref{eq:22}.

Hence by proceeding similarly as before we conclude that for the user defined by the embedding $e(u_i+u)$, the item $x$ is strictly better than any hott-topic in  $[d] - \{ i\}$.

So the best item for this user cannot be any hott-topic in $[d] - \{ i\}$. Consequently we conclude that the best item for the user defined by the embedding $e(u_i+u)$ must be $i$.

Now we look at how much the item $i$ is suboptimal for the user $u$.

\begin{align}
    {\fl{R}}_{u,j} - {\fl{R}}_{u,i}
    &=  {\fl{R}}_{u,j} - {\fl{R}}_{u,x} + {\fl{R}}_{u,x} - {\fl{R}}_{u,i}\\
    &\le 5k_{\ell-1} + {\fl{R}}_{u,x} - {\fl{R}}_{u,i}\\
    &= 5k_{\ell-1} + e(u_i+u)^T e'(x-i) + {\fl{R}}_{u_i,i} - {\fl{R}}_{u_i,x}\\
    &\le 10 k_{\ell-1} + e(u_i+u)^T e'(x-i)\\
    &\le 10 k_{\ell-1}, \label{eq:x3}
\end{align}
where we used Eq. \eqref{eq:x1} and \eqref{eq:x2} and the fact that item $i$ is best for the user defined by the embedding $e(u_i+u)$.

To make further progress, we assume the following: $|\tilde {\fl{R}}^\ell_{u,i} - {\fl{R}}_{u,i}| \le k_{\ell-1}$ with probability at-least $1-\delta$.

This is trivially true for the first epoch in stage 2 which is $\ell_0+1$. Because at epoch $\ell_0+1$ we have that $|\tilde {\fl{R}}^{\ell_0+1}_{y,h} - {\fl{R}}_{y,h}| \le k_{\ell_0}$ for all users $y$ and all items $h$. This follows from the fact that until epoch $\ell_0$, we estimate the entire reward matrix globally.

We will proceed to show that $|\tilde {\fl{R}}^{\ell+1}_{u,i} - {\fl{R}}_{u,i}| \le k_{\ell}$ as well.

Recall that $\pi_u(1) = j$. We have

\begin{align}
    \tilde {\fl{R}}^{\ell}_{u,\tilde \pi_u(1)} - \tilde {\fl{R}}^{\ell}_{u,i}
    &= \tilde {\fl{R}}^{\ell}_{u,\tilde \pi_u(1)} - {\fl{R}}_{u,\tilde \pi_u(1)} +  {\fl{R}}_{u,\tilde \pi_u(1)} -  {\fl{R}}_{u,\pi_u(1)} + {\fl{R}}_{u,\pi_u(1)} - {\fl{R}}_{u,i} + {\fl{R}}_{u,i} - \tilde {\fl{R}}^{\ell}_{u,i}\\
    &\le \tilde {\fl{R}}^{\ell}_{u,\tilde \pi_u(1)} - {\fl{R}}_{u,\tilde \pi_u(1)} + {\fl{R}}_{u,\pi_u(1)} - {\fl{R}}_{u,i} +  {\fl{R}}_{u,i} - \tilde {\fl{R}}^{\ell}_{u,i}\\
    &\le 2k_{\ell-1} + {\fl{R}}_{u,\pi_u(1)} - {\fl{R}}_{u,i}\\
    &\le 12 k_{\ell-1},
\end{align}
where in the last line we used Eq.\eqref{eq:x3}.

Thus we conclude that the expansion step in Line 3 of Algorithm \ref{alg:clique} ensures that $i \in \cT_u^\ell$.

Now to prove the induction hypothesis that $|\tilde {\fl{R}}^{\ell+1}_{u,i} - {\fl{R}}_{u,i}| \le k_{\ell}$, we look at the for loop in Line 15 of Algorithm \ref{alg:full} when processing the cluster $i$ and during the epoch $\ell$. Note that $u \in \ca{U}$ by virtue of the expansion step. Consequently in the matrix estimation step in Line 17 of Algorithm \ref{alg:full}, we end up estimating the entry ${\fl{R}}_{u,i}$ with errror at-most $k_{\ell}$. This proves the induction hypothesis.

The best item for user $u$ was already present in its candidate item set before the start of the Algorithm \ref{alg:clique}. This candidate item set was constructed by accumulating all items with empirical sub-optimality level of $3k_{\ell-1}$. Now noting that the expansion step cannot eliminate the best item of the user $u$ concludes the proof. 

\end{proof}

\begin{lemma}(stage 2 best item survival) \label{lem:stage3-best}
Consider an epoch $\ell$ within stage 2 (so $\ell \ge \ell_0+1$). Suppose for any user $u$, $
\pi_u(1) \in \cT_u^{\ell}$. Then with probability at-least $1-\delta$ we have that $\pi_u(1) \in \cT_u^{\ell+1}$.
\end{lemma}
\begin{proof}
Updating the candidate item set in Stage 2 happens at two places. In Line 13 of Algorithm \ref{alg:full} via call to \texttt{ExpanditemSets} as well as in Line 11. In Lemma \ref{lem:clique-best-item}, we showed that the call to Algorithm \ref{alg:clique} does not eliminate the best item of any user. 

So we only need to argue that the step in Line 11 of Algorithm \ref{alg:full} also does not eliminate the best item of any user. To do this, consider a user $v$ with $\pi_v(1) = i$. This implies that the user $v$ is connected to $u_i$ and $v \in \ca{U}$ in Line 16 of Algorithm \ref{alg:full} (when processing cluster $i$) since the best item is not eliminated prior to epoch $\ell$. Since by Lemma \ref{lem:clique-best-item}, the users that are connected to user $u_i$ have the item $i$ in their candidate item set, we conclude that $i \in \cT_{C_i}^\ell$ at Line 14 of Algorithm \ref{alg:full}. This implies that after the call to \texttt{ESTIMATE} in Line 17, we must have that $|{\fl{R}}_{u,i}^{\ell+1} - {\fl{R}}_{u,i}| \le k_\ell$ via Proposition \ref{prop:mat1}.

Now putting it all together yields,

\begin{align}
    \tilde {\fl{R}}^{\ell+1}_{u, \tilde \pi_u(1)} - \tilde {\fl{R}}^{\ell+1}_{u,\pi_u(1)}
    &= \tilde {\fl{R}}^{\ell+1}_{u, \tilde \pi_u(1)} - {\fl{R}}_{u, \tilde \pi_u(1)} + {\fl{R}}_{u,\pi_u(1)} - \tilde {\fl{R}}^{\ell+1}_{u,\pi_u(1)} + {\fl{R}}_{u, \tilde \pi_u(1)} - {\fl{R}}_{u,\pi_u(1)}\\
    &\le 2k_\ell, 
\end{align}

with probability at-least $1-\delta$ where the last line follows by Proposition \ref{prop:mat1}. Hence the best item is not eliminated with high probability within Stage 2.

\end{proof}

\begin{lemma}(instantaneous regret in Stage 2)
    Consider an epoch $\ell$ withing Stage 2 of Algorithm \ref{alg:full}. Then with probability at-least $1-\delta$, we have that for any user $u$,
    \begin{align}
        {\fl{R}}_{u,\pi_u(1)} - {\fl{R}}_{u,\rho_u(t)}
        &\le k_{\ell-1},
    \end{align}
    where $t$ is a time in epoch $\ell$ and $\rho_u(t)$ is the item recommended to user $u$ as per Algorithm \ref{algo:estimate}.
\end{lemma}
\begin{proof}
Let $\cT_u^\ell$ be the item set maintained by Algorithm \ref{alg:full} before the call in Line 17.  We know from Lemma \ref{lem:stage3-best} that $\pi_u(1) \in \cT_u^\ell$ with probability at-least $1-\delta$. Further due to Line 11 in Algorithm \ref{alg:full} and Line 3 in Algorithm \ref{alg:clique}, we have that $\tilde {\fl{R}}^\ell_{u,\pi_u(1)} - \tilde {\fl{R}}^\ell_{u,x} \le 12 k_{\ell-1}$ for any $x \in \cT_u^\ell$.

So

\begin{align}
    {\fl{R}}_{u,\pi(1)} - {\fl{R}}_{u,\rho_u(t)}
    &= {\fl{R}}_{u,\pi(1)} - \tilde {\fl{R}}^{\ell}_{u,\pi(1)} + \tilde {\fl{R}}^\ell_{u,\rho_u(t)} -  {\fl{R}}_{u,\rho_u(t)} + \tilde {\fl{R}}^{\ell}_{u,\pi(1)} - \tilde {\fl{R}}^{\ell}_{u,\tilde \pi(1)} + \tilde {\fl{R}}^{\ell}_{u,\tilde \pi(1)} - \tilde {\fl{R}}^\ell_{u,\rho_u(t)}\\
    &\le {\fl{R}}_{u,\pi(1)} - \tilde {\fl{R}}^{\ell}_{u,\pi(1)} + \tilde {\fl{R}}^\ell_{u,\rho_u(t)} -  {\fl{R}}_{u,\rho_u(t)} + \tilde {\fl{R}}^{\ell}_{u,\tilde \pi(1)} - \tilde {\fl{R}}^\ell_{u,\rho_u(t)}\\
    &\le 2k_{\ell-1} + 12 k_{\ell-1}\\
    &= 14k_{\ell-1},
\end{align}
with probability at-least $1-\delta$ where we used the matrix completion guarantee from proposition \ref{prop:mat1}.

\end{proof}

The following result is a direct consequence of Assumption \ref{assum:simplicity}.

\begin{prop} \label{ass:users-lb}
For each hott-topic $i \in [d]$ there is at-least $\kappa \s{M}$ users that have their best item as $i$, where $\kappa = O(1)$.
\end{prop}

\begin{lemma} \label{lem:rounds}
Consider an epoch $\ell$ in Stage 2. Then Algorithm \ref{algo:estimate} when called in Line 17 of Algorithm \ref{alg:full} executes for a total of $\tilde O(\max(1,\s{N}/\s{M})/k_{\ell}^2)$ rounds.
\end{lemma}
\begin{proof}
Consider a hott-topic $i \in [d]$. By Lemma \ref{lem:stage3-best}, the best item for each user survives in its candidate item set. Hence we have that $|\ca{U}| \ge \kappa \s{M}$ by Proposition \ref{ass:users-lb}, where $\ca{U}$ is as in Line 16 of Algorithm \ref{alg:full} while processing cluster $C_i$.

Next, we proceed to bound the number of rounds run by \texttt{ESTIMATE} when it is called at Line 17 of Algorithm \ref{alg:full}.

Here we have $|\ca{U}| \ge \kappa \s{M}$ and $|\cT_{D_i}^\ell| \le \s{N}$. In the case where $\min(|\ca{U}|, |\cT_{D_i}^\ell|) = |\cT_{D_i}^\ell|$, then by Proposition \ref{prop:mat1} we have that the number of rounds run by \texttt{ESTIMATE} is $\tilde O(1/k_{\ell}^2)$. In the other case where $\min(|\ca{U}|, |\cT_{D_i}^\ell|) = |\ca{U}|$, then the number of rounds is $\tilde O((\s{N}/|\ca{U}|)/k_{\ell}^2) = \tilde O((\s{N}/\s{M})/k_{\ell}^2)$ since $|\ca{U}| \ge \kappa \s{M}$. This concludes the proof.

\end{proof}

\begin{defn} \label{def:small-gap}
The minimum sub-optimality gap is defined as:
\begin{align}
    \Delta := \min_{u \in [\s{M}]} \min_{j \in [\s{N}], j\neq \pi_u(1)} {\fl{R}}_{u, \pi_u(1)} - {\fl{R}}_{u,j}.
\end{align}
\end{defn}

\begin{lemma}[Stage 2 regret bound] \label{lem:stage2-regret}
Assume $\Delta$ is as defined in Definition \ref{def:small-gap}. Then the total regret incurred in Stage 2 is bounded by $\tilde  O(\max(1,\s{N}/\s{M})/\Delta)$ with probability at-least $1-\delta \log T$.
\end{lemma}
\begin{proof}
Due to Lemma \ref{lem:subopt}, with probability at-least $1-\delta$, the candidate item sets will only contain the best item for each user whenever $5k_\ell \le \Delta$ (see Definition \ref{def:small-gap}.  Consequently we can say that for epochs $\ell \ge \log(1/(5\Delta))$ the algorithm do not suffer any regret at all.

So we only need to bound the regret till the epoch $\ell^* := \log(1/(5\Delta))$. Again by Lemma \ref{lem:subopt} and Lemma \ref{lem:rounds}, the regret incurred in an epoch $\ell$ within stage 2 is bounded by $\tilde O(\max(1,\s{N}/\s{M})/k_{\ell})$ with probability greater than $1-\delta$. Thus the total regret across all epochs in stage 2 is

\begin{align}
    {\fl{R}}_{\text{stage 2}}
    &= \sum_{\ell = \ell_0+1}^{\ell^*} \tilde O(\max(1,\s{N}/\s{M})/k_{\ell})\\
    &= \tilde  O(\max(1,\s{N}/\s{M})/k_{\ell^*})\\
    &= \tilde  O(\max(1,\s{N}/\s{M})/\Delta),
\end{align}
where line 2 is due to the halving progression of $k_\ell$. Taking a union bound of the failure probability across all epochs concludes the proof.

\end{proof}

We are now ready to present the proof of the main result in Setting 1.

\textbf{Proof of Theorem \ref{thm:first}}
The proof of the theorem is an immediate consequence of Lemmas \ref{lem:stage1-regret} and \ref{lem:stage2-regret}. With probability at-least $1-2\delta \log T$, we can bound the total regret by $\widetilde{O}\Big(\max\Big(1,\s{\s{N}}\s{\s{M}}^{-1}\Big)\Big(\Delta_{\s{hott}}^{-2}+\Delta^{-1}\Big)+1\Big)$. Now doing a change of variables from $2\delta \log T$ to $\delta$ yields the theorem.

\subsection{Proof of Lemma \ref{lem:prelim_1}}

\begin{proof}[Proof of Lemma \ref{lem:prelim_1}]
Let us fix a particular user $u\in [\s{M}]$. Assume without loss of generality that $\fl{v}^{(1)}$ is the item in $\ca{C}$ with the highest reward i.e. 
$
    \langle \fl{u},\fl{v}^{(1)} \rangle \ge \langle \fl{u}, \fl{v}\rangle \text{ for all }\fl{v}\in\ca{C}\setminus \{\fl{v}^{(1)}\}
$
In that case for any item vector $\fl{z}=\sum_{\fl{v}\in \ca{C}}\alpha_{\fl{v}}\fl{v}\in \ca{V}\setminus \ca{C}$ (non-negative coefficients $\{\alpha_{\fl{v}}\}_{\fl{v}\in \ca{C}} $satisfying $\sum_{\fl{v}\in \ca{C}}\alpha_{\fl{v}}=1$ and $\alpha_{\fl{v}} > 0$ for some $\fl{v}\in \ca{C}\setminus \{\fl{v}^{(1)}\}$), we must have that $\langle \fl{u},\fl{v}^{(1)} \rangle > \langle \fl{u}, \fl{z}\rangle$. Hence, $\s{argmax}_{j\in [\s{N}]} \fl{R}_{uj} \in \ca{A}$ and a similar argument also shows $\s{argmin}_{j\in [\s{N}]} \fl{R}_{uj} \in \ca{A}$. 
\end{proof}

\subsection{Proof of Theorem \ref{thm:second}}

  Recall that our goal is to gradually eliminate columns from $\Pi_r([\s{N}])$ and converge to the set of hott items induced by $\ca{A}$ in as few rounds as possible. We propose a successive elimination algorithm (Algorithm \textsc{DeterminantElim}) where at each phase indexed by $\ell$, we will maintain a surviving set of $r$ columns denoted by $\ca{B}^{(\ell)}\subseteq \Pi_r([\s{N}])$ (initialized by $\ca{A}^{(1)}=\Pi_r([\s{N}])$). Since at each round, we need to recommend items to every single user, we will design a recommendation scheme such that in phase $\ell$, with high probability, we can ensure that for every $r$-column in $\ca{B}^{(\ell)}$, we have noisy observations corresponding to at least $n_{\ell}$ users that are sampled uniformly from $[\s{M}]$. Concurrently, we also need to ensure that the number of sufficient rounds in phase $\ell$ satisfies the condition in equation \eqref{eq:sampling}. 
 We now prove our key lemma using the probabilistic method:
 
\begin{lemma}\label{lem:sampling}
Fix $n_{\ell}>0,0\le \delta \le 1$. Let $\ca{B} = \bigcup_{\ca{J}\in \ca{B}^{(\ell)}}\ca{J}$  be the set of columns that are surviving in phase $\ell$. For each user $u\in [\s{M}]$, suppose we map $d_{\ell}$ distinct items in $\ca{B}$ sampled uniformly at random for recommendation in phase $\ell$. With probability at least $1-\delta$, for every $r$-column in $\ca{B}^{(\ell)}$, we obtain noisy observations from at least $n_{\ell}$ distinct users provided 
\begin{align}\label{eq:sampling}
    \frac{d_{\ell}}{\s{N}-d_{\ell}} \ge \Big(\s{M}^{-1}(12\log\delta^{-1}+r\log \s{N}+2n_{\ell})\Big)^{1/r}
\end{align}
\end{lemma}

\begin{proof}[Proof of Lemma \ref{lem:sampling}]
The items sampled for user $u$ is denoted by $\ca{T}_u^{(\ell)}\subseteq \ca{B}$.
Fix a particular $r$-column $\ca{J}\in \ca{B}^{(\ell)}$. 
Suppose $s=\left|\ca{B}\right|$.
The probability that for any user $u\in [\s{M}]$, all items in $\ca{J}$ have been sampled i.e. $\ca{J}\subseteq \ca{T}_u^{(\ell)}$ is 
\begin{align}
    p_{\ell} \triangleq \Pr(\ca{J}\subseteq \ca{T}_u^{(\ell)}) = \frac{{s-r\choose d_{\ell}-r}}{{s \choose d_{\ell}}} \le \frac{{s \choose d_{\ell}-r}}{{s \choose d_{\ell}}} \le  \Big(\frac{d_{\ell}}{s-d_{\ell}}\Big)^r.
\end{align}
Notice that the expected number of users for which $\ca{J}\subseteq \ca{T}_u^{(\ell)}$ is  $\s{M}p_{\ell}$. 
Suppose $p_{\ell}$ is set such that $\s{M}p_{\ell}\ge 2n_{\ell}$.   
Since the $d$ items sampled for each user are independent, the probability of the event $\ca{E}^{(\ell)}_{\ca{J}}$ that $\ca{J}$ is not a subset of $\ca{T}_u^{(\ell)}$ for at least $n_{\ell}$ users is given by (let us denote the random variable corresponding to the number of users in $[\s{M}]$ for which $\ca{J}\in \ca{T}_u^{(\ell)}$ by $X$)
\begin{align}
    \Pr(X\le n_{\ell}) \le \Pr(\left|X-\bb{E}X\right|\ge \s{M}p_{\ell}-n_{\ell}) \le \Pr(\left|X-\bb{E}X\right|\ge \frac{\s{M}p_{\ell}}{2}) \le 2\exp(-\s{M}p_{\ell}/12) 
\end{align}

% \begin{align}
% \Pr(\ca{E}^{(\ell)}_{\ca{J}}) = \sum_{j=0}^{n_{\ell}-1} {\s{M}\choose j}p_{\ell}^{j} (1-p_{\ell})^{\s{M}-j} \le n_{\ell} {\s{M} \choose n_{\ell}} p_{\ell}^{n_{\ell-1}}.
% \end{align}

Hence, the probability that $\ca{E}_{\ca{J}^{(\ell)}} $ is true for some $\ca{J}\in \ca{B}^{(\ell)}$ is given by 
\begin{align} \Pr\Big(\bigcup_{\ca{J}\in \ca{B}^{(\ell)}}\ca{E}^{(\ell)}_{\ca{J}}\Big) \le 2\left|\ca{B}^{(\ell)}\right|\exp(-\s{M}p_{\ell}/12)  \le
2\s{N}^r\exp(-\s{M}p_{\ell}/12).
\end{align}
 Therefore, for some value of $0<\delta<1$, if we have
\begin{align}
    \s{M}p_{\ell} \ge 12\log\delta^{-1}+r\log \s{N}+2n_{\ell}  
\end{align}
then with probability $1-\delta$, the event $\bigcap_{\ca{J}\in \ca{B}^{(\ell)}}\ca{E}^{(\ell)}_{\ca{J}}$ is true. Hence, the above condition is satisfied if 
\begin{align}
    \frac{d_{\ell}}{s-d_{\ell}} \ge \Big(\s{M}^{-1}(12\log\delta^{-1}+r\log \s{N}+2n_{\ell})\Big)^{1/r}
\end{align}
\end{proof}

We define the event $\ca{E}^{(\ell)}$ which is true when for $d_{\ell}$ satisfying eq. \eqref{eq:sampling}, for every $r$-column $\ca{J}$ in $\ca{B}^{(\ell)}$, there will exist at least $n_{\ca{J}}\ge n_{\ell}$ users $u\in [\s{M}]$ for which $\ca{J}\in \ca{T}_u^{(\ell)}$. From Lemma \ref{lem:sampling}, we know that $\ca{E}^{(\ell)}$ is true with probability at least $1-\delta$. 
Furthermore, due to our sampling approach being invariant to a permutation of the users, conditioned on the value of $n_{\ca{J}}$, the corresponding $n_{\ca{J}}$ users are sampled uniformly at random from $[\s{M}]$. In Alg. \textsc{DeterminantElim}, in phase $\ell$, once the sets $\ca{T}_u^{(\ell)}$ have been sampled, every item in $\ca{T}_u^{(\ell)}$ is recommended to user $u$ at least $2\sigma^2$ times (See Step 4).  
From our recommendation strategy, the total number of rounds sufficient in phase $\ell$ to make all the recommendations necessary in our scheme is $2\sigma^2d_{\ell}$. %Here, an important thing to note is that in phase $\ell$, on sampling the set $\ca{T}_u^{(\ell)}$, each item $z$ in $\ca{T}_u^{(\ell)}$ is recommended to user $u$ $2\sigma^2$ times. However, in our setting, we need to recommend $r$ distinct items to every user at each round - therefore, if we need to recommend item $z$ to user $u$, we instead recommend some set $\ca{J}\in \ca{B}^{(\ell)}$ such that $z\in \ca{J}$.
Now, for any $r$-column $\ca{J}$ in $[\s{N}]^r$, define 
\begin{align}
    \mu_{\ca{J}}\triangleq \frac{1}{{\s{M} \choose r}} \sum_{\ca{I}\subset [\s{M}]\mid \left|\ca{I}\right|=r} \s{det}^2(\fl{R}_{\ca{I},\ca{J}}) <  \frac{1}{{\s{M} \choose r}} \sum_{\ca{I}\subset [\s{M}]\mid \left|\ca{I}\right|=r} \s{det}(\fl{R}_{\ca{I},\ca{A}}) \triangleq \mu_{\ca{A}}.
\end{align}
Recall $c_{\s{avg}}\triangleq \frac{1}{{\s{M} \choose r}} \sum_{\ca{I}\subset [\s{M}]\mid \left|\ca{I}\right|=r} \s{det}^2(\fl{U}_{\ca{I}})$ is  the average of  determinants of the user features taken $r$ at a time - implying that $\mu_{\ca{J}}=c_{\s{avg}} \s{det}^2(\fl{V}_{\ca{J}})$. Moreover, $c_{\max}=\s{det}^2(\fl{V}_{\ca{A}})$ corresponds to the matrix of hott vectors that has the largest unsigned determinant.
For $\ca{J}\in \ca{B}^{(\ell)}$, let $\ca{H}_{\ca{J}}$ be the set of users for which items in $\ca{J}$ were recommended in phase $\ell$ according to our sampling scheme in the first and second components respectively. We compute an estimate $\widetilde{\mu}_{\ca{J}}$ of $\mu_{\ca{J}}$ in the following way:
\begin{align}\label{eq:estimate_2}
    \widetilde{\mu}_{\ca{J}}\triangleq \frac{1}{{n_{\ca{J}} \choose r}} \sum_{\ca{I}\subset \ca{H}_{\ca{J}}\mid \left|\ca{I}\right|=r} \s{det}(\widetilde{\fl{R}}^{(\ell,1)}_{\ca{I},\ca{J}})\cdot \s{det}(\widetilde{\fl{R}}^{(\ell,2)}_{\ca{I},\ca{J}})
\end{align}
where the matrices $\widetilde{\fl{R}}^{(\ell,1)},\widetilde{\fl{R}}^{(\ell,2)}$ are described in Step 5 of Alg. \textsc{DeterminantElim}. 
We can now characterize the concentration property of the estimate $\widetilde{\mu}_{\ca{J}}$ for all $\ca{J}\in \ca{B}^{(\ell)}$:

\begin{lemma}\label{lem:conc}
Fix $n_{\ell}$ and condition on $\ca{E}^{(\ell)}$. Then, in phase $\ell$ of Alg. \textsc{DeterminantElim}, we must have that for all $\ca{J}\in \ca{B}^{(\ell)}$, $\widetilde{\mu}_{\ca{J}}\in [\mu_{\ca{J}}-\epsilon_{\ell},\mu_{\ca{J}}+\epsilon_{\ell}]$ for $\epsilon_{\ell} = \sqrt{(2^{r+1}r^{1+r/2})\log(\s{N}^r\s{T})n_{\ell}^{-1}}$ with probability at least $1-O(\s{T}^{-2})$.
\end{lemma}

\begin{proof}[Proof of Lemma \ref{lem:conc}]

We also condition on the event $\ca{F}^{(\ell)}$ that the random variables corresponding to the mean noise for all rounds $t$ in phase $\ell$ is bounded within $[-10\sigma\sqrt{b_{\ell}^{-2}\log(\s{MT})},10\sigma\sqrt{b_{\ell}^{-2}\log(\s{MT})}]$ - event $\ca{F}^{(\ell)}$ is true with probability at least $1-\s{T}^{10}$ (sub-gaussian concentration properties).
We set $b_{\ell}^2=100\sigma^2\log\s{MT}$ such that noise random variables are bounded within $[-1,1]$ with probability at least $1-\s{T}^{-10}$. 

    Note that there are two sources of randomness in $\widetilde{\mu}_{\ca{J}}$ - the first source of randomness is due to the noise added in the observations and the second source is due to the randomness in $\ca{H}_{\ca{J}}$. Note that 
\begin{align}
&\bb{E}\widetilde{\mu}_{\ca{J}}\mid \ca{E}^{(\ell)},\ca{F}^{(\ell)} =\bb{E}_{\ca{H}_{\ca{J}}}\frac{1}{{n_{\ca{J}} \choose r}} \sum_{\ca{I}\subset \ca{H}_{\ca{J}}\mid \left|\ca{I}\right|=r} \bb{E}\Big[\s{det}(\widetilde{\fl{R}}^{(\ell,1)}_{\ca{I},\ca{J}})\cdot \s{det}(\widetilde{\fl{R}}^{(\ell,2)}_{\ca{I},\ca{J}})\mid \ca{H}_{\ca{J}},\ca{E}^{(\ell)},\ca{F}^{(\ell)}\Big] \\
    &=\bb{E}_{\ca{H}_{\ca{J}}}\frac{1}{{n_{\ca{J}} \choose r}} \sum_{\ca{I}\subset \ca{H}_{\ca{J}}\mid \left|\ca{I}\right|=r} \bb{E}\Big[\s{det}(\widetilde{\fl{R}}^{(\ell,1)}_{\ca{I},\ca{J}})\mid \ca{H}_{\ca{J}},\ca{E}^{(\ell)},\ca{F}^{(\ell)}\Big]\cdot \bb{E}\Big[\s{det}(\widetilde{\fl{R}}^{(\ell,2)}_{\ca{I},\ca{J}})\mid \ca{H}_{\ca{J}},\ca{E}^{(\ell)},\ca{F}^{(\ell)}\Big] \\    &=\bb{E}_{\ca{H}_{\ca{J}}}\big[\frac{1}{{n_{\ca{J}} \choose r}} \sum_{\ca{I}\subset \ca{H}_{\ca{J}}\mid \left|\ca{I}\right|=r}\s{det}^2(\fl{R}_{\ca{I},\ca{J}})\mid \ca{E}^{(\ell)}\Big].
\end{align}
We use the fact that conditioning on $\ca{E}^{(\ell)}$ only modifies the distribution of $n_{\ca{J}}$.
Here, we also used the fact that $\bb{E}\s{det}(\fl{Z})=\s{det}(\bb{E}\fl{Z})$ for zero mean independent noise which is the case even after conditioning on $\ca{F}^{(\ell)}$ (due to symmetricity of the sub-gaussian random variables).
Recall that conditioned on $n_{\ca{J}}$, the set $\ca{H}_{\ca{J}}$ is sampled uniformly at random from $\Pi_{n_{\ca{J}}}[\s{M}]$ i.e. sets of all distinct users of size $n_{\ca{J}}$. Therefore, we can simplify to conclude that $\bb{E}\widetilde{\mu}_{\ca{J}}\mid \ca{E}^{(\ell)},\ca{F}^{(\ell)} =\mu_{\ca{J}}$. By the same set of arguments, we can show that $\bb{E}\widetilde{\mu}_{\ca{J}}=\mu_{\ca{J}}$ - hence, we prove that conditioning on the events $\ca{E}^{(\ell)},\ca{F}^{(\ell)}$ do not change the expected value of the estimate $\widetilde{\mu}_{\ca{J}}$.

Next, we show that conditioned on the events $\ca{E}^{(\ell)},\ca{F}^{(\ell)}$ and on $n_{\ca{J}}$, the estimate $\widetilde{\mu}_{\ca{J}}$ is concentrated around its mean $\mu_{\ca{J}}$ - for which we are going to use McDiarmid's inequality \cite{vershynin2020high}[Theorem 2.9.1]. The lemma that we state below is a minor modification of the standard McDiarmid's inequality for uniform sampling without replacement.

\begin{lemma}[McDiarmid's inequality for sampling without replacement]\label{lem:mcdiarmid}
 Consider a function $f:\ca{X}^n \rightarrow \bb{R}$ which satisfies the bounded difference property i.e. there exists a constant $c>0$ such that for all $\fl{x},\fl{y}\in \ca{X}^n$ satisfying for some $i\in [n]$ 1) $\fl{y}_j=\fl{x}_j$ for $j\neq i$ and 2) $\fl{y}_i\neq\fl{x}_i$, we have
 \begin{align}
     \left|f(\fl{x})-f(\fl{y})\right|\le c.
 \end{align}
 In that case, for random variables $
 \fl{x}_1,\fl{x}_2,\dots,\fl{x}_n\in \ca{X}$ sampled uniformly at random without replacement, we must have for $\fl{x}=(\fl{x}_1,\fl{x}_2,\dots,\fl{x}_n)$,
 \begin{align}
     \Pr\Big(\left|f(\fl{x})-\bb{E}\fl{x}\right|\ge \epsilon\Big) \le 2\exp\Big(-\frac{2\epsilon^2}{nc^2}\Big).
 \end{align}
\end{lemma}

\begin{rmk}
    Note that the proof of Lemma \ref{lem:mcdiarmid} is a trivial modification of the proof to the standard McDiarmid's inequality. For the standard McDiarmid's inequality, the random variables $\fl{x}_1,\fl{x}_2,\dots,\fl{x}_n$ are sampled independently from $\ca{X}$ - this is used crucially for the fact that the probability of any instantiation of the random variables $\fl{x}_{i+1},\dots,\fl{x}_n$ does not depend on the values assigned to the random variables $\fl{x}_1,\dots,\fl{x}_i$. In fact this is the only (but crucial) use of the independence. For the uniform sampling without replacement case, the independence does not hold - yet,   the probability of any instantiation of the random variables $\fl{x}_{i+1},\dots,\fl{x}_n$ is simply $1/{\left|\mathcal{X}\right|-i \choose n-i}$ and therefore does not depend on the values assigned to the random variables $\fl{x}_1,\dots,\fl{x}_i$. Therefore all remaining steps hold true.
\end{rmk}

In our setting, we let $\ca{H}_{\ca{J}}$ (with $\left|\ca{H}_{\ca{J}}\right|=n_{\ca{J}}\ge n_{\ell}$) be the set of users sampled for items in $\ca{J}$. In that case, we define the function of the observations corresponding to the items in $\ca{J}$ and users in $\ca{H}_{\ca{J}}$ to be the RHS in eq. \eqref{eq:estimate}. Note that because of how we set the value of $\sigma^2$, conditioned on the event $\ca{F}^{(\ell)}$, each entry of $\widetilde{\fl{R}}^{(\ell,1)},\widetilde{\fl{R}}^{(\ell,2)}$ are bounded between $[-1,2]$. Therefore the maximum determinant $\s{det}_{\max}(r)$ of an $r\times r$ sub-matrix of either of $\widetilde{\fl{R}}^{(\ell,1)},\widetilde{\fl{R}}^{(\ell,2)}$ must be bounded from above by $2^r r^{r/2}$ (Hadamard's determinant inequality). Therefore, on altering the set of observations corresponding to a particular user in $\ca{H}_{\ca{J}}$ by a different set of observations, conditioned on the event $\ca{F}^{(\ell)}$, the change in RHS in  eq. \eqref{eq:estimate} is at most  $${n_{\ca{J}}\choose r-1}({n_{\ca{J}}\choose r})^{-1}2^r r^{r/2} \le 2^r r^{1+r/2}(n_{\ca{J}}-r)^{-1}\le 2^{r+1} r^{1+r/2}n^{-1}_{\ca{J}}$$ 
where we use the fact that $n_{\ca{J}}\ge n_{\ell}\ge 2r$.
Hence, we can conclude by application of McDiarmid's inequality 
\begin{align}  \Pr\Big(\left|\widetilde{\mu}_{\ca{J}}-\mu_{\ca{J}}  \right| \ge \epsilon_{\ell} \mid \ca{E}^{(\ell)},\ca{F}^{(\ell)}\Big)\le  2\exp\Big(-\frac{2n_{\ca{J}}\epsilon_{\ell}^2}{2^{r+1}r^{1+r/2}}\Big).
\end{align}
Therefore, by setting 
\begin{align}
    \epsilon_{\ell} = \sqrt{\frac{(2^{r+1}r^{1+r/2})\log(\s{N}^r\s{T})}{n_{\ell}}},
\end{align}
with probability at least $2\s{N}^{-r}\s{T}^{-2}$, the estimate $\widetilde{\mu}_{\ca{J}}$ must belong to the interval $[\mu_{\ca{J}}-\epsilon_{\ell},\mu_{\ca{J}}+\epsilon_{\ell}]$.
Therefore, conditioning on events $\ca{E}^{(\ell)},\ca{F}^{(\ell)}$ by taking a union bound over all possible $r$-columns of $[\s{N}]$, we can conclude that for all surviving $r$-columns $\ca{J}$ in $\ca{B}^{(\ell)}$, with our algorithm, the estimates $\widetilde{\mu}_{\ca{J}}$ will be within $\epsilon_{\ell}$ of their mean $\mu_{\ca{J}}$ for our choice of $\epsilon_{\ell}$. Since the event $(\ca{F}^{(\ell)})^{c}$ happens with an order-wise smaller probability ($O(\s{T}^{-10})$ in particular), we can ignore the conditioning on $\ca{F}^{(\ell)}$.

\end{proof}

Let us define $\ca{G}^{(\ell)}$ if $\widetilde{\mu}_{\ca{J}}\in [\mu_{\ca{J}}-\epsilon_{\ell},\mu_{\ca{J}}+\epsilon_{\ell}]$ in phase $\ell$ for our choice of $\epsilon_{\ell}$ in Lemma \ref{lem:conc}.  Next, we show the following lemma (slightly modified from \cite{kveton2017stochastic}) to characterize instantaneous regret:
\begin{lemma}\label{lem:instant}
   Suppose for a user $u\in [\s{M}]$, we recommend a set of $r$ distinct items at round $t$ (denoted by $\ca{J}\equiv \{\rho_u(t,j)\}_{j\in [r]}$). In that case, we will have
   \begin{align}
    \fl{R}_{u\pi_u(1)}-\max_{j\in [r]}\fl{R}_{u\rho_u(t,j)} \le 6r^{5/2}\Big(\frac{\s{det}^2(\fl{V}_{\ca{A}})-\s{det}^2(\fl{V}_{\ca{J}})}{\s{det}^2(\fl{V}_{\ca{A}})}\Big).
    \end{align}
\end{lemma}

\begin{proof}[Proof of Lemma \ref{lem:instant}]  
Consider any permutation $\nu:[r]\rightarrow[r]$. We denote the $r$-dimensional basis vectors by $\{\fl{e}_i\}_{i\in [r]}$. Also, let $\fl{V}_{\ca{A}}$ be the matrix whose rows are formed by the hott vectors $\{\fl{v}^{(a_j)}\}_{j\in [r]}$. Then any item vector $\fl{v}'$ can be written as $\fl{V}_{\ca{A}}^{\s{T}}\fl{s}$ where $\fl{s}$ is a vector with non-negative entries that sum up to at most $1$. For the $j^{\s{th}}$ item, we will have $\fl{v}^{(j)}=\fl{V}_{\ca{A}}^{\s{T}}\fl{s}^{(j)}$.
Note that we have for a user $u$,
\begin{align}
    \fl{R}_{u\pi_u(1)}-\max_{j\in [r]}\fl{R}_{u\rho_u(t,\nu(j))} &= (\fl{u}^{(u)})^{\s{T}}\fl{v}^{(\pi_u(1))}-\max_{j\in [r]} (\fl{u}^{(u)})^{\s{T}}\fl{v}^{(\rho_u(t,\nu(j)))} \\
    &= \max_{j\in [r]}(\fl{u}^{(u)})^{\s{T}}\fl{v}^{(a_j)}-\max_{j\in [r]} (\fl{u}^{(u)})^{\s{T}}\fl{v}^{(\rho_u(t,\nu(j)))} \\
    & \le \max_{j\in [r]} \left|(\fl{u}^{(u)})^{\s{T}}\fl{v}^{(a_j)}-\max_{j\in [r]} (\fl{u}^{(u)})^{\s{T}}\fl{v}^{(\pi_u(t,\rho(j)))}\right| \\
    &\le \sum_{j\in [r]} \left|(\fl{u}^{(u)})^{\s{T}}\fl{v}^{(a_j)}- (\fl{u}^{(u)})^{\s{T}}\fl{v}^{(\rho_u(t,\nu(j)))}\right|\\
    &= \left|\left|(\fl{u}^{(u)})^{\s{T}}\fl{V}_{\ca{A}}^{\s{T}}\right|\right|_2 \sum_{j\in [r]} \left|\left| \Big(\fl{e}_i-\fl{s}^{\rho_u(t,\nu(j))}\Big)\right|\right|_2 \\
    &\le 2r\sum_{j\in [r]}\lr{\fl{e}_i-\fl{s}^{\rho_u(t,\nu(j))}}_2.
\end{align}
At this point, we can simply use the lemma in \cite{kveton2017stochastic} which says the following:
\begin{lemma}[Lemma 4 in \cite{kveton2017stochastic}]

There must exist a permutation $\nu:[r]\rightarrow[r]$ such that 
\begin{align}
    \sum_{j\in [r]}\lr{\fl{e}_i-\fl{s}^{\rho_u(t,\nu(j))}}_2 \le 6r^{3/2}\Big(\frac{\s{det}^2(\fl{V}_{\ca{A}})-\s{det}^2(\fl{V}_{\ca{J}})}{\s{det}^2(\fl{V}_{\ca{A}})}\Big)
\end{align}
\end{lemma}
Combining the two statements above, we obtain the desired result.
\end{proof}

Now we are ready to prove Thm \ref{thm:second} using the intermediate Lemmas \ref{lem:sampling}, \ref{lem:conc} and \ref{lem:instant}.

\begin{proof}[Proof of Theorem \ref{thm:second}]
    Recall that in the first phase $\ell=1$, we initialize the set of $r$-columns by $\Pi_{r}([\s{N}])$.
In each phase $\ell$, we eliminate subsets of $r$-columns. Suppose at beginning of phase $\ell$, we have $\ca{B}^{(\ell)}$ to be the set of surviving $r$-columns and for each $\ca{J}\in \ca{B}^{(\ell)}$, we compute $\widetilde{\mu}_{\ca{J}}$, an estimate of $\mu_{\ca{J}}$. For the subsequent phase $\ell+1$, we compute $\ca{B}^{(\ell+1)}$ in the following way:
\begin{align}\label{eq:eliminate}
    \ca{B}^{(\ell+1)} = \{\ca{J}\in \ca{B}^{(\ell)}\mid \widetilde{\mu}_{\ca{J}} \ge \max_{\ca{J}'\in \ca{B}^{(\ell})}\widetilde{\mu}_{\ca{J}'}-2\epsilon_{\ell}\}.
\end{align}
First, we show that if $\ca{A}$, the set of hott item vectors belong to $\ca{B}^{(\ell)}$, then $\ca{A}$ should also belong to $\ca{B}^{(\ell+1)}$ conditioned on the event $\ca{G}^{(\ell)}$. To see this, note that $\mu_{\ca{A}}>\mu_{\ca{J}}$ for all $\ca{J}\neq \ca{A}$. Since, conditioned on the event $\ca{G}$, we have $\widetilde{\mu}_{\ca{J}}\in [\mu_{\ca{J}}-\epsilon_{\ell},\mu_{\ca{J}}+\epsilon_{\ell}]$, by triangle inequality, it must happen that $\ca{A}\in \ca{B}^{(\ell+1)}$. Furthermore, for all $\ca{J}\in \ca{B}^{(\ell+1)}$, it must happen that (Lemma \ref{lem:conc})
\begin{align}
    \mu_{\ca{J}} \ge \mu_{\ca{A}}-2\epsilon_{\ell} \implies \s{det}^2(\fl{V}_{\ca{J}})\ge \s{det}^2(\fl{V}_{\ca{A}}) -\frac{2\epsilon_{\ell}}{c_{\s{avg}}}.
\end{align}
Hence, in all rounds $t$ in phase $\ell$, the regret incurred by user $u$ (Lemma \ref{lem:instant}) is bounded by
\begin{align}
    \fl{R}_{u\pi_u(1)}-\max_{j\in [r]}\fl{R}_{u\rho_u(t,\rho(j))} \le \frac{12r^{5/2}\epsilon_{\ell}}{c_{\s{avg}}c_{\s{max}}}
\end{align}
In phase $\ell \ge 1$, we set $\epsilon_{\ell}=2^{-\ell}$ and $\epsilon_0=1$. Hence, we have $n_{\ell}= C(r)2^{2\ell}$ where $C(r)=2^{r+1}r^{1+r/2}$. Thus, a sufficient value for $d_{\ell}$ and the number of rounds in phase $\ell$ is (Lemma \ref{lem:sampling})
\begin{align}
    d_{\ell} = O\Big(\frac{\s{N}}{\s{M}^{1/r}}\cdot C(r)2^{2\ell}\Big) \text{ and }2\sigma^2d_{\ell} = O\Big(\frac{\sigma^2\s{N}\log(\s{MT})}{\s{M}^{1/r}}\cdot C(r)2^{2\ell}\Big) \text{ respectively}.
\end{align}
Hence, if the total number of phases is $m=O(\log \s{T})$ within which all sub-optimal $r$-columns are eliminated by the algorithm, then by combining the above statements, the regret can be computed as 
\begin{align}
    \s{Reg}_{\s{Simplified}}(\s{T}) &\le \sum_{\ell=1}^{m} 2\sigma^2d_{\ell}\epsilon_{\ell}+\s{T}\Pr(\cup_{\ell}(\ca{E}^{(\ell)})^c\cup_{\ell}(\ca{F}^{(\ell)})^c\cup_{\ell} (\ca{G}^{(\ell)})^c) \\
    &= O\Big(1+\frac{\sigma^2\s{N}\log(\s{MT})}{\s{M}^{1/r}}\cdot \frac{2^{r+1}r^{(r+7)/2}}{c_{\s{avg}}c_{\s{max}}\Delta_{\s{det}}}\Big).
\end{align}
\end{proof}

\section{Discussion on Feasibility of Assumptions}\label{app:assum}

Note that the proof analysis of Theorem \ref{thm:first} generalizes easily (without almost any changes) if the number of hott items is $d>r$ as well. We stress that only for ease of exposition, we have considered the number of hott vectors to be $r$.

In this section we are going to show evidence that Assumptions \ref{assum:hott} and \ref{assum:matrix2} are true together for a large number of rank-$r$ reward matrices $\fl{R}\in \bb{R}^{\s{M}\times \s{N}}$ with the number of hott topics to be $d=2r$. In particular, suppose we decompose $\fl{R}=\fl{X}^{\s{T}}\fl{Y}$ where $\fl{X}\in \bb{R}^{r\times \s{M}}$, $\fl{Y}\in \bb{R}^{r\times \s{N}}$ are the transposes of  user and item embedding matrices respectively. We consider the generative model where 
\begin{enumerate}
    \item All entries of $\fl{X}$ are independently generated from $\ca{N}(0,1)$.
    \item The first $r$ columns of $\fl{Y}$ are the identity matrix multiplied by a factor of $10\log (r\s{N})$. The second $r$ columns of $\fl{Y}$ are the identity matrix multiplied by a factor of $-10\log (r\s{N})$.
    All entries in the remaining $\s{N}-2r$ columns of $\fl{Y}$ are uniformly generated from the interval $[0,1]$. 
\end{enumerate}

For the first part of our proof, we will show guarantees on the SVD of $\fl{X}=\fl{U}\f{\Sigma}\fl{V}^{\s{T}}$ where $\fl{U}\in \bb{R}^{r\times r},\fl{V}\in \bb{R}^{\s{M}\times r}$ are orthonormal matrices (singular vectors) and $\f{\Sigma}\in \bb{R}^{r\times r}$ is a diagonal matrix corresponding to the singular values of $\fl{X}$. More specifically, we have the following lemmas from \cite{https://doi.org/10.48550/arxiv.2301.07040} - however, we include the proofs for the sake of completeness. We will denote $\fl{X}_{\mid \ca{S}}$ to be the matrix $\fl{X}$ restricted to the columns in $\ca{S}$ and similarly, we denote $\fl{X}_{\ca{S}}$ to denote the matrix $\fl{X}$ restricted to the rows in $\ca{S}$. For a matrix $\fl{X}$, we will write $\lambda_j(\fl{X})$ to denote the $j^{\s{th}}$ singular value of $\fl{X}$.

\begin{lemma}\label{lem:sss}[\cite{https://doi.org/10.48550/arxiv.2301.07040}]
Fix $\gamma >0$.
If $\fl{x}^{\s{T}}\fl{X}_{\mid \ca{S}}\fl{X}_{\mid \ca{S}}^{\s{T}}\fl{x} \ge \alpha \gamma  r\lambda_1^2/\s{M}$ for a subset $\ca{S}\subseteq [\s{M}], |\ca{S}|=\gamma  r$ for all unit vectors $\fl{x}\in \bb{R}^{  r}$, then the minimum eigenvalue of $\fl{V}_{\ca{S}}^{\s{T}}\fl{V}_{\ca{S}} \ge \alpha\gamma   r/\s{M}$. 
\end{lemma}

\begin{proof}
 If $\fl{x}^{\s{T}}\fl{X}_{\mid \ca{S}}\fl{X}_{\mid \ca{S}}^{\s{T}}\fl{x} \ge \alpha \gamma  r\lambda_1^2/\s{M}$ for a subset $\ca{S}\subseteq [\s{M}], |\ca{S}|=\gamma  r$ for all unit vectors $\fl{x}\in \bb{R}^{  r}$, then the minimum eigenvalue of $\fl{V}_{\ca{S}}^{\s{T}}\fl{V}_{\ca{S}} \ge \alpha\gamma   r/\s{M}$. To see this, note  $\fl{X}_{\mid \ca{S}}=\fl{U}\f{\Sigma}\fl{V}^{\s{T}}_{\ca{S}}$ implying that $\fl{V}^{\s{T}}_{ \ca{S}}=(\fl{U}\f{\Sigma})^{-1}\fl{X}_{\mid \ca{S}}$. Hence, $\fl{V}^{\s{T}}_{\ca{S}}\fl{V}_{\ca{S}} = (\fl{U}\f{\Sigma})^{-1}\fl{X}_{\mid \ca{S}}\fl{X}_{\mid \ca{S}}^{\s{T}}(\fl{U}\f{\Sigma})^{-\s{T}}$ implying that $(\fl{V}^{\s{T}}_{\ca{S}}\fl{V}_{\ca{S}})^{-1} = (\fl{U}\f{\Sigma})^{\s{T}}(\fl{X}_{\mid\ca{S}}\fl{X}_{\mid\ca{S}}^{\s{T}})^{-1}(\fl{U}\f{\Sigma})$. Taking the operator norm on both sides, we have  $\lambda_{\min}(\fl{V}^{\s{T}}_{\ca{S}}\fl{V}_{\ca{S}}) \ge \lambda_1^{-2} \lambda_{\min}(\fl{X}_{\mid \ca{S}}\fl{X}^{\mid \s{T}}_{\ca{S}})$ implying that $\fl{x}^{\s{T}}\fl{V}_{\ca{S}}^{\s{T}}\fl{V}_{\ca{S}}\fl{x} \ge \alpha\gamma  r/\s{M}$.
\end{proof}

\begin{lemma}\label{lem:X_analysis}
Suppose $  r \ll \s{M}$ and further, the entries of $\fl{X}$ are generated independently according to $\ca{N}(0,1)$.
In that case, it must happen that with high probability that (for $\gamma=16\log\s{M}$)
\begin{enumerate}
    \item $\lambda_1(\fl{X})/\lambda_r(\fl{X})=O(1)$
    \item Hence $\lr{\fl{V}}_{2,\infty} \le 16\sqrt{\frac{  r\log \s{M}}{\s{M}}}$
    \item $\lambda_r(\fl{V}_{\ca{S}}) \ge \beta$ for some constant $\beta>0$ where $\ca{S}$ corresponds to a nice subset of users. 
\end{enumerate}
\end{lemma}

\begin{proof}
 We must have $\sqrt{\s{M}}-\sqrt{  r}-t\le \lambda_{  r}(\fl{X}) \le \lambda_{1}(\fl{X}) \le \sqrt{\s{M}}+\sqrt{  r}+t$ w.p. at least $1-2e^{-t^2/2}$ implying that $\sqrt{\s{M}}/2 \le \lambda_{  r} \le \lambda_1 \le 2\sqrt{\s{M}}$; hence we must have $\lambda_1/\lambda_{  r}=O(1)$ w.p. at least $1-O(e^{-\s{M}})$. Moreover, we have $\fl{X}^{\s{T}}\fl{X}=\fl{V}\f{\Sigma}^2\fl{V}^{\s{T}}$. Clearly, we must have $\lr{\fl{X}}_{\infty,2}\lambda_{  r}^{-1} \le \lr{\fl{V}}_{2,\infty} \le \lr{\fl{X}}_{\infty,2}\lambda_1^{-1}$. For any column $\fl{X}_{\mid i}$, we have $\lr{\fl{\fl{X}_{\mid i}}}_2^2$ is a chi-squared random variable with $  r$ degrees of freedom. Using standard  concentration inequalities for chi-squared random variables, we have $\lr{\fl{\fl{X}_{\mid i}}}_2 \le 8\sqrt{  r\log \s{M}}$ with probability at least $1-\s{M}^{-2}$. By taking a union bound over all $i\in [\s{M}]$, we have $\lr{\fl{X}}_{\infty,2} \le 8\sqrt{  r\log \s{M}}$ w.p. at least $1-\s{M}^{-1}$. Hence $\lr{\fl{V}}_{2,\infty} \le 16\sqrt{\frac{  r\log \s{M}}{\s{M}}}$. Therefore, when $\fl{X}$ is a random Gaussian matrix, the first part of Lemma  holds true and the second part of Lemma holds true with $\mu=O(\log \s{M})$ with high probability.

On the other hand, for a subset $\ca{S}\subseteq [\s{M}]$ such that $\ca{S}$ corresponds to a \textit{nice subset of users} we must have
that $\left|\ca{S}\right| \ge c\s{M}$ for some constant $c>0$ (with Assumption \ref{assum:simplicity} being true). In that case, we have the minimum singular value of $\fl{X}_{\mid \ca{S}}$ to be at least $\sqrt{c\s{M}}-\sqrt{r}-t$ w.p. at least $1-2e^{-t^2/2}$. Taking $t=\sqrt{4  r\log \s{M}}$, we must have the minimum singular value of $\fl{X}_{\mid \ca{S}}$ to be at least $\sqrt{c\s{M}}/2$ with probability at least $1-2e^{-2  r\log\s{M}}$. Taking a union bound over all nice subsets (at most $2^{\s{C}}$ of them where $\s{C}$ is the number of clusters), we must have for all unit norm vectors $\fl{x}\in \bb{R}^{  r}$
\begin{align*}
\fl{x}^{\s{T}}\fl{X}_{\mid \ca{S}}\fl{X}_{\mid \ca{S}}^{\s{T}}\fl{x} \ge \frac{c^2\s{M}}{4} = \frac{c^2\s{M}\cdot 4\s{M}}{16\s{M}} \ge \frac{c^2\lambda_1^2}{16}      
\end{align*}
with probability at least $1-O(e^{-\s{M}})$ implying that $\alpha \ge 1/16$. Combining with Lemma \ref{lem:sss}, we can prove the final conclusion of the lemma.
\end{proof}

\textit{Conjecture:} \textit{We will have that the minimum singular value of any sub-matrix $\fl{Y}_{\mid \ca{S}}$ restricted to the columns in $\ca{S}$ of size $\left|\ca{S}\right|=\gamma r$ to be at least $\frac{r\gamma}{4}$}. An open problem to a similar effect has been stated in \cite{bhojanapalli2014universal}.

\textit{Discussion:}
 We take $\gamma=16\log \s{N}$.
Again, for a subset $\ca{S}\subseteq [\s{N}], |\ca{S}|=\gamma  r$, we must have the minimum singular value of $\fl{Y}_{\mid \ca{S}}$ to be at least $\sqrt{  r}(\sqrt{\gamma}-1)-t$ with probability at least $1-2e^{-t^2/2}$. Taking  $t=2\sqrt{r\log \s{N}}$, we must have the minimum singular value of $\fl{Y}_{\mid \ca{S}}$ to be at least $\sqrt{  r\gamma}/2$ with probability at least $1-2e^{-2  r\log \s{N}}$. 

Taking a union bound over all such subsets $\ca{S}$ of $[\s{M}]$ with size $\left|\ca{S}\right|=\gamma   r$, we have the failure probability to be ${N \choose 16r\log \s{N}}e^{-2r\log \s{N}}$ which becomes large - more precisely ${
\s{N} \choose 16r\log \s{N}}\approx e^{r\log^2 \s{N}}$ that is, the factor in the exponent is  larger by a multiplicative factor of $\log \s{N}$ than the tail probability. At this point we conjecture that such an analysis is weak as the union bound does not capture the dependencies among the sets due to non-empty intersections. In fact, even theoretically, for most sets, the minimum singular value is going to be large. We leave a tight analysis of this case as a future work. 
However, if the failure probability for this event is indeed smaller than $1$, then  
we must have for all unit norm vectors $\fl{x}\in \bb{R}^{  r}$
\begin{align*}
\fl{x}^{\s{T}}\fl{Y}_{\mid \ca{S}}\fl{Y}_{\mid \ca{S}}^{\s{T}}\fl{x} \ge \frac{  r\gamma}{4} = \frac{  r\gamma\cdot 4\s{N}}{16\s{N}} \ge \frac{  r\gamma \lambda_1^2}{16\s{N}}      
\end{align*}
with high probability at least implying that $\alpha \ge 1/16$. Combining with Lemma \ref{lem:sss}, we can conclude that if the conjecture is true, then $\lambda_r(\widehat{\fl{V}}_{\ca{S}}) \ge \sqrt{\frac{r\gamma}{16\s{N}}}$.

Recall that $\fl{P}=\fl{X}^{\s{T}}\fl{Y}=\fl{V}\f{\Sigma}\fl{U}^{\s{T}}\widehat{\fl{U}}\widehat{\f{\Sigma}}\widehat{\fl{V}}$. We take a further SVD of $\f{\Sigma}\fl{U}\widehat{\fl{U}}\widehat{\f{\Sigma}}=\widetilde{\fl{U}}\widetilde{\f{\Sigma}}\widetilde{\fl{V}}$. Therefore, we get that $\fl{P}=\fl{V}\widetilde{\fl{U}}\widetilde{\f{\Sigma}}\widetilde{\fl{V}} \widehat{\fl{V}}$. Note that $\fl{V}\widetilde{\fl{U}}$ and $\widetilde{\fl{V}} \widehat{\fl{V}}$ are orthonormal matrices and therefore $\fl{V}\widetilde{\fl{U}}\widetilde{\f{\Sigma}}\widetilde{\fl{V}}\widehat{\fl{V}}$ is a valid SVD of $\fl{P}$. The properties of Lemmas \ref{lem:X_analysis} and \ref{lem:Y_analysis} are preserved in the modified singular value matrices $\fl{V}\widetilde{\fl{U}}$ and $\widetilde{\fl{V}}\widehat{\fl{V}}$ respectively. Since the condition numbers of both $\fl{X},\fl{Y}$ are constants, therefore the condition number of $\fl{P}$ is also a constant.
Moreover, the $\lr{\fl{V}\widetilde{\fl{U}}}_{2,\infty} \le \lr{\fl{V}}_{2,\infty}$ and $\lr{\widetilde{\fl{V}} \widehat{\fl{V}}}_{2,\infty} \le \lr{\widehat{\fl{V}}}_{2,\infty}$. Finally, for any set $\ca{S}\subseteq [\s{M}]$, we will have $\lambda_r(\fl{V}_{\ca{S}}\widetilde{\fl{U}}) = \lambda_r(\fl{V})$ and a similar property also holds for $\widehat{\fl{V}}$.

\begin{lemma}\label{lem:Y_analysis}
Assume that the above conjecture is true. Suppose $r \ll \s{N}$ and further, the entries of last $\s{N}-2r$ columns of $\fl{Y}$ are generated independently according to $\ca{N}(0,1)$.
In that case, it must happen with high probability that (suppose SVD of $\fl{Y}=\widehat{\fl{U}}\widehat{\f{\Sigma}}\widehat{\fl{V}}$)
\begin{enumerate}
    \item  $\lambda_1(\fl{Y})/\lambda_r(\fl{Y})=O(1)$.
    \item Hence $\lr{\widehat{\fl{V}}}_{2,\infty} \le 16\sqrt{\frac{  r\log \s{M}}{\s{M}}}$ 
\end{enumerate}
\end{lemma}

We move on to a analysis for the matrix $\fl{Y}$ which is very similar to the matrix $\fl{X}$. Note that Lemma \ref{lem:sss} holds true for the matrix $\fl{Y}$ as well with $\s{M}$ replaced by the number of items $\s{N}$. We now show the following lemma:

First of all, we have that the eigen-spectrum of $\fl{Y}$ is same as the eigen-spectrum of $\fl{Y}^{\s{T}}$.
Notice that for any unit vector $\fl{z}$, we have that 
\begin{align*}  \sum_{i=1}^{\s{N}}\fl{z}^{\s{T}}\fl{Y}_i\fl{Y}_i^{\s{T}}\fl{z} =20 \log (r\s{N})+\sum_{i=2r}^{\s{N}}\fl{z}^{\s{T}}\fl{Y}_i\fl{Y}_i^{\s{T}}\fl{z}.
\end{align*}
Combining with the analysis of Lemma \ref{lem:X_analysis}, we must have that, with probability at least $1-2e^{-t^2/2}$, 
\begin{align*}
  \sqrt{\s{N}}-\sqrt{r}-t+\sqrt{20\log(r\s{N})}  \le  \lambda_r(\fl{Y}) \le \lambda_1(\fl{Y}) \le \sqrt{\s{N}}+\sqrt{r}+t+\sqrt{20\log(r\s{N})}.
\end{align*}
Therefore, we can again conclude that that $\lambda_1(\fl{Y})/\lambda_r(\fl{Y})=O(1)$ with probability at least $1-O(e^{-\s{M}})$.
Furthermore, we also have that $\lr{\fl{Y}}_{\infty,2}\le 8\sqrt{r\log \s{N}}$ with probability at least $1-\s{N}^{-1}$. Again, if the SVD of the matrix $\fl{Y}=\widehat{\fl{U}}\widehat{\f{\Sigma}}\widehat{\fl{V}}$ where $\widehat{\fl{U}},\widehat{\fl{V}}$ are orthonormal matrices, then $\lr{\widehat{\fl{V}}}_{2,\infty}\le \lr{\fl{Y}}_{\infty,2}\lambda_1(\fl{Y})^{-1}$. Since $\lambda_1(\fl{Y}) \ge \frac{
\sqrt{\s{N}}}{4}$ with probability at least $1-e^{-\s{N}}$, we have that $\lr{\widehat{\fl{V}}}_{2,\infty}\le 32\sqrt{\frac{r\log \s{N}}{\s{N}}}$.

Consider some $\s{N}'=c\s{N}$ for a suitable constant $c>1$ which will be decided later. Suppose we randomly generate $\s{N}'$ $r$-dimensional vectors $\fl{y}^1,\fl{y}^2,\dots,\fl{y}^{\s{N}'}$ such that every entry of each of these vectors is independently generated from $\ca{N}(0,1)$. Now, for 
some fixed $\epsilon>0$,
consider an $\epsilon$-net of the $r$-dimensional unit ball $\ca{S}_{r,\epsilon}$ such that for any unit vector $\fl{x}\in \bb{R}^r$, there must exist $\fl{z}\in \ca{S}_{r,\epsilon}$ such that $\lr{\fl{x}-\fl{z}}_2 \le \epsilon$. Now, for any unit fixed vector $\fl{z}\in \ca{S}_{r,\epsilon}$ and any $i\in [\s{N}']$, we have that the random variable 
$\fl{z}^{\s{T}}\fl{y}^i \sim \ca{N}(0,1)$. Further note that the random variables  $\{\fl{z}^{\s{T}}\fl{y}^i\}_{i=1}^{\s{N}'}$ are independent. Now, for any $X\sim \ca{N}(0,1)$, we must have that for any constant $\gamma<1$
\begin{align*}
    \Pr(|X|\le \gamma) \le \sqrt{\frac{2}{\pi}}\cdot \gamma.
\end{align*}
Therefore, for a fixed $\gamma>0$, consider the random variables $J_1,J_2,\dots,J_{\s{N}'}$ where $J_i=\mathds{1}[\left|\fl{z}^{\s{T}}\fl{y}^i\right| \le \gamma]$. Clearly, we have that 
\begin{align*}
    \sum_{i\in [\s{N}']}\bb{E}J_{i} = \sqrt{\frac{2}{\pi}}\cdot \gamma \s{N}' 
\end{align*}
Since the variables $\{J_i\}_i$ are indicator random variables, we can directly use Chernoff bound to conclude that
with probability at least $1-\exp(-\s{N})$, we have 
$\sum_{i} J_i\le \sqrt{\frac{8}{\pi}}\cdot \gamma \s{N}'$.
We will choose $\epsilon=1/3$ and it is well-known that there exists an $\epsilon-$net $\ca{S}_{r,\epsilon}$ satisfying $\left|\ca{S}_{r,\epsilon}\right|\le 9^r$.
Therefore, if we take a union bound over all $9^r$ vectors in $\ca{S}_{r,\epsilon}$, then we must have with probability $1-2^r\exp(-\s{N})$, 
\begin{align*}
    \sum _{\fl{z}\in \ca{S}_{r,1/3}}\sum_{i \in [\s{N}']} \mathds{1}[\left|\fl{z}^{\s{T}}\fl{y}^i\right| \le \gamma] \le  \sqrt{\frac{2}{\pi}} \cdot 9^r\cdot \gamma \s{N}' 
\end{align*}

Recall that $r$ is a small constant and therefore, $9^r$ is a constant as well. Hence, 
we choose the constant $\gamma>0$ such that $ \sqrt{\frac{2}{\pi}} \cdot 9^r\cdot \gamma \le 10^{-1}$. Therefore, we have a subset $\ca{T}\subseteq \{\fl{y}^1,\dots\fl{y}^{\s{N}'}\}$ of least $9\s{N}'/10$ vectors such that for all $\fl{y}\in \ca{T}$, we have for the constant $\gamma>0$ and $\epsilon=1/3$ 
\begin{align*}    \left|\fl{z}^{\s{T}}\fl{y}\right| \ge \gamma \text{ for all }\fl{z}\in \ca{S}_{r,\epsilon}.
\end{align*}
We will use the vectors in  $\ca{T}$ to be the set of vectors for constructing the last $\s{N}-2r$ columns of $\fl{Y}$.

Now, consider any set $\ca{S}\subseteq \ca{T}$. Consider the matrix $\fl{H}$ whose columns are given by the vectors in $\ca{S}$. In that case, we have for all $\fl{z}\in \ca{S}_{r,\epsilon}$,
\begin{align*}
    \fl{z}^{\s{T}}(\fl{I}-\fl{y}\fl{y}^{\s{T}})\fl{z} \le 1-\gamma.
\end{align*}

\end{document}